\documentclass{article} 
\usepackage{iclr2027_conference,times}

\usepackage{amsmath,amsfonts,bm}

\def\eqref#1{equation~\ref{#1}}

\def\1{\bm{1}}

\def\eps{{\epsilon}}

\def\rmI{{\mathbf{I}}}

\def\rmP{{\mathbf{P}}}

\def\vzero{{\bm{0}}}

\def\veps{{\bm{\eps}}}

\def\vc{{\bm{c}}}

\def\vg{{\bm{g}}}

\def\vu{{\bm{u}}}
\def\vv{{\bm{v}}}
\def\vw{{\bm{w}}}
\def\vx{{\bm{x}}}

\def\vz{{\bm{z}}}

\DeclareMathAlphabet{\mathsfit}{\encodingdefault}{\sfdefault}{m}{sl}
\SetMathAlphabet{\mathsfit}{bold}{\encodingdefault}{\sfdefault}{bx}{n}

\def\gN{{\mathcal{N}}}

\usepackage{hyperref}
\usepackage{url}
\usepackage{booktabs}
\usepackage{amsmath,amssymb,mathtools,amsfonts,bm}
\usepackage{microtype}
\usepackage{nicefrac}
\usepackage{xcolor}
\usepackage{graphicx}
\usepackage{subcaption}
\usepackage{multirow}
\usepackage{wrapfig}
\usepackage{algorithm}
\usepackage[noend]{algpseudocode}
\usepackage{enumitem}
\usepackage{array}
\usepackage{makecell}
\usepackage{amsthm, amssymb}
\usepackage{thmtools}
\usepackage{thm-restate}
\usepackage{pifont}

\def\eqref#1{(\ref{#1})}

\title{Manifold-Constrained 
Initial Noise Optimization for Efficient Generative Model Alignment}

\author{Jinho Chang \& Jong Chul Ye \thanks{Corresponding author.} 
\\
Kim Jaechul Graduate School of AI\\
KAIST\\
Seoul, South Korea \\
\texttt{\{jinhojsk515,jong.ye\}@kaist.ac.kr}
}

\iclrfinalcopy 
\begin{document}

\maketitle

\begin{abstract}
Recent advances in distillation and flow-map models have enabled deterministic one- or few-step generation for high-quality data, facilitating a new branch of reward alignment approaches that directly optimize the initial noise from a Gaussian distribution. However, most existing initial-noise optimization methods rely on first-order gradient information, which is either inapplicable or suffers from instability and inefficiency in black-box reward scenarios.
Here, we introduce ZeNOVA (\textbf{Ze}roth-order \textbf{N}oise \textbf{O}ptimization with \textbf{V}alue \textbf{A}nnealing), a stable and efficient initial noise alignment method in a gradient-free manner. Specifically, we address existing algorithms’ major challenge in black-box scenarios through annealed soft-value guidance, manifold-constrained hyperspherical Langevin dynamics, and Metropolis-Hastings jumping.
Extensive experiments on image and video generative models show that ZeNOVA outperforms all evaluated zeroth-order baselines by optimizing the initial noise toward higher rewards substantially more stably while exploiting the geometry of the Gaussian prior, demonstrating its practical applicability to various black-box reward alignment.
\end{abstract}

\section{Introduction}
Reward alignment of a pre-trained generative model aims to align its output distribution to admit higher external reward while minimizing its deviation from the original prior. A wide range of fine-tuning based approaches are presented grounded on several theoretical support, including entropy-regularized reinforcement learning~\citep{black2024training,fan2023dpok,liu2025flowgrpo,he2025tempflowgrpotimingmattersgrpo, zheng2026diffusionnft}, GFlowNets~\citep{zhang2025dagdb,liu2025resnabladb,liu2025nablar2d3}, and stochastic optimal control~\citep{domingo-enrich2025adjoint}. Inference-time methods employing various techniques such as {diffusion} posterior sampling~\citep{chung2023diffusion, tfg, freedom, mpgd, lgd} and path integral control~\citep{pmlr-v235-huang24g, Kim_2025_ICCV} have also been proposed to implement the same effect without additional training.

However, these works have to assume stochastic sampling trajectories involving many intermediate steps for their methods to analytically admit the reward-tilted distribution and thereby satisfy both the prior and high reward.
Consequently, naively extending the results of these works to deterministic generative models is either limited or does not guarantee its behavior as theoretically designed.
Moreover, recent advances in generative modeling have led to the development of few-step or one-step models including distilled models~\citep{yin2024one, yin2024improved, sauer2024adversarial} and flow-map models~\citep{song2023consistency, geng2025mean}, which reduces the effect of control in between intermediate sampling process.

\begin{figure}[t!]
     \centering
     \includegraphics[width=\linewidth]{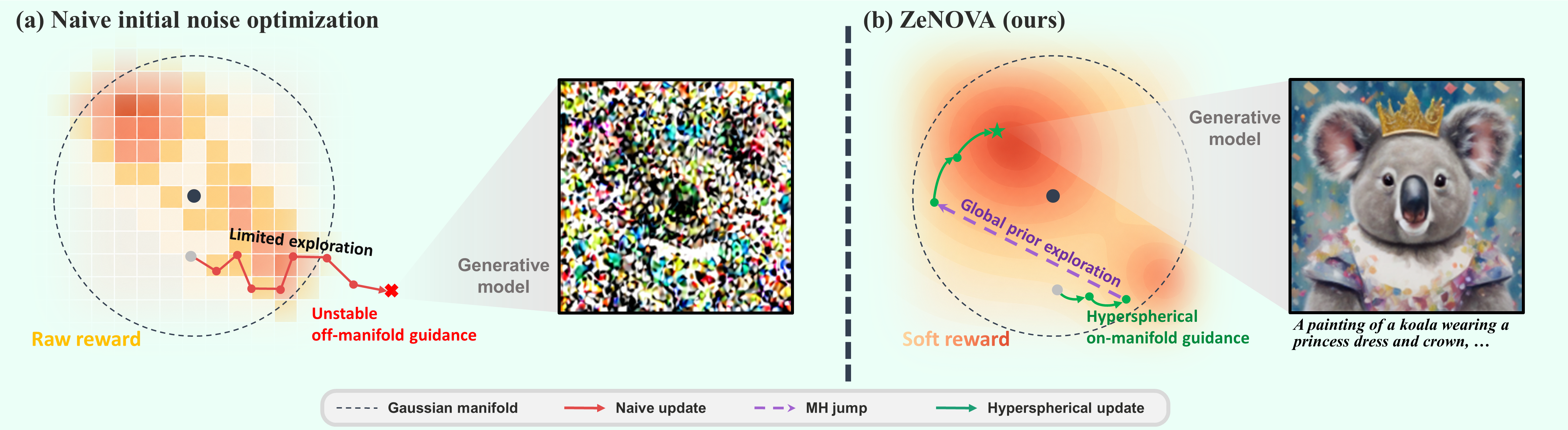}
     \caption{\textbf{Overview of ZeNOVA.} Naive initial-noise optimization suffers from limited exploration and unstable updates under inaccurate zeroth-order guidance from raw reward signal. ZeNOVA combines annealed soft-value guidance, hyperspherical updates, and Metropolis-Hastings jumping to enable stable local refinement and global exploration while respecting the Gaussian prior geometry.}
     \label{fig:main}
     \vspace{-3mm}
 \end{figure}

These recent advances facilitate a separate approach of initial-noise optimization as a practical route to inference-time reward alignment; a pretrained deterministic generator can produce different outputs by changing its Gaussian input, without updating its parameters~\citep{eyring2024reno, min2026origen, yoon2026psisampler}. Most existing approaches, however, obtain their optimization directions by differentiating through both the generator and the reward computation. This requirement limits their use when the reward is non-differentiable or available only through black-box evaluations, and can impose substantial computational overhead even when gradients are accessible. Estimating directions from reward queries offers a gradient-free alternative, but useful alignment also requires these estimates to support stable and efficient optimization in the high-dimensional noise space. In particular, seeking higher rewards should account for the Gaussian {noise} input geometry on which the generator was trained, which still remains challenging with optimizations in gradient-free manner.

Here, we propose \textbf{ZeNOVA}, a stable and efficient gradient-free initial-noise alignment method toward black-box objectives. 
Specifically, as illustrated in Figure~\ref{fig:main}, we identify the major challenges faced by existing initial noise optimization algorithms in black-box scenarios, and address them through three contributions of annealed soft-value guidance, manifold-constrained hyperspherical Langevin dynamics, and Metropolis-Hastings jumping.
Together, these components require only forward evaluations of the generator and reward model, making ZeNOVA applicable to black-box alignment without backpropagation. Experiments on image and video generative models show that ZeNOVA achieves higher rewards than all evaluated zeroth-order baselines while maintaining more stable initial-noise optimization, supporting its practical use across diverse reward alignment tasks.

\section{Preliminaries}

\textbf{Reward-tilted Distribution.}
When $p^{\mathrm{ref}}$ and $p^\theta$ denote the distribution from the pretrained model and the fine-tuned model with parameters $\theta$, Entropy-regularized RL formulates the alignment objective as balancing expected terminal reward against divergence from the reference process with a balancing parameter $\lambda>0$:
\begin{equation}
\label{eq:softrl}
    \mathcal{L}(\theta) = \mathbb{E}_{\vx_0 \sim p_0^\theta} \left[ r(\vx_0) \right]-\lambda D_{\mathrm{KL}} \left( p^\theta \,\|\, p^{\mathrm{ref}} \right),
\end{equation}
In many algorithms for solving Eq.~\eqref{eq:softrl} in diffusion and flow-matching models~\citep{fan2023dpok}, the trajectory-level KL admits a decomposition into a sum or integral of local policy divergences along the sampling trajectory. When $t\in[0,1]$ denotes the noise timestep and $\vx_0$ denotes the clean sample, it is known that the optimal policy for Eq.~\eqref{eq:softrl} induces the following marginal distribution~\citep{uehara2024fine} with \emph{Soft value function} $V(\vx_t)$, which represents the expected reward and KL penalty term over the optimal policy from $\vx_t$:
\begin{align}
\label{eq:tilted}
    p_t^\star(\vx_t) \propto p_t^{\mathrm{ref}}(\vx_t) e^{\frac{1}{\lambda}V(\vx_t)}, 
   \quad
    V(\vx_t) 
    :=\lambda\log \mathbb{E}_{p_t^{\mathrm{ref}}}\!\left[\exp\!\left(\frac{r(\vx_0)}{\lambda}\right)\middle|\vx_t\right]
\end{align}
Substituting the relationship between the output of a diffusion/flow-matching model and the score function~\citep{song2020score}, Eq.~\eqref{eq:tilted} shows that the difference between the optimal and reference models shall be proportional to the gradient of this value function $\nabla_{\vx_t}V(\vx_t)$.

Among several mathematical frameworks, 
Path-Integral Control(PIC)~\citep{pathintegral,pmlr-v235-huang24g, Uehara2024understanding} allows us to express the optimal policy and soft value function through weighted samples of the reference dynamics:
\begin{align}
\label{eq:optimal_control}
\nabla_{\vx_t}V(\vx_t)
&=\lim_{dt\to 0}\frac{\lambda}{g(t)dt}
\frac{\mathbb{E}_{P^{\mathrm{ref}}}\!\left[e^{r(\vx_0)/\lambda}d\vw_{t}\mid\vx_t\right]}
{\mathbb{E}_{P^{\mathrm{ref}}}\!\left[e^{r(\vx_0)/\lambda}\mid\vx_t\right]}
\end{align}
This representation weights sampled noise increments by terminal desirability, allowing the control to be estimated from samples of the reference model and terminal reward evaluations alone.

{
\textbf{Reward-tilted Noise Distribution.} 
Consider reward alignment of the input initial noise fed to the generative models, without assuming access to reward gradients.
Let $\vz\sim p_0:=\gN(0,\rmI_d)$ denote the $d$-dimensional input noise, and $G_\theta(\vz,\vc)$ be a generator conditioned on $\vc$,
leading to the reward in noise space and the reward-titled noise distribution:
\begin{equation}
    R(\vz) := r(G_\theta(\vz,\vc)),\quad     p^\star(\vz) \propto p_0(\vz)\exp\left(R(\vz) / \lambda\right) \label{eq:target_tilted}
\end{equation}
To perform optimization of the initial noise $\vz$, the simplest approach is to search for a maximizer of $R(\vz)$ over the entire support $\mathbb{R}^d$. But this is not only prone to mode collapse toward a few high-reward regions, it also disregards the Gaussian prior distribution of $\vz$ on which the latter generative model $G_\theta$ is trained, leading to reward hacking and degradation of the image quality. 
While several empirical regularizers were suggested to mitigate this~\citep{eyring2024reno, mono}, 
when $R$ is differentiable with respect to its input $\vz$, sampling from~\eqref{eq:target_tilted} is straightforward because its score function is tractable.
Hence, the following Langevin dynamics has $p^\star$ as its stationary distribution.
\begin{equation}
\label{eq:grad_langevin}
    d\vz
    = \left(-\vz+\frac{1}{\lambda}\nabla_{\vz}R(\vz)\right)dt
    + \sqrt{2}\,d\vw
\end{equation}
This observation is utilized in several gradient-based initial-noise alignment algorithms~\citep{min2026origen, yoon2026psisampler}.
}

\section{Method}
\label{sec:method}
{Our interest is a more challenging setting in which $R$ can be evaluated but $\nabla_{\vz}R$ is unavailable. }
Here, 
we {first} show why a direct zeroth-order extension of gradient-based Langevin optimization is inadequate, and {then} describe the proposed contributions of ZeNOVA to address those challenges.

\subsection{Challenges for Black-box Objectives.}
\label{subsec:zo_failure}

\textbf{A natural zeroth-order baseline.}
For a small perturbation scale $\delta>0$, a classic gradient estimation method with function evaluations according to Stein's Lemma~\citep{Stein1981} is
\begin{equation}
\label{eq:local_zo_grad}
    \mathbb{E}_{\veps}\left[\frac{1}{\delta}R(\vz+\delta\veps)\veps\right] = \nabla_{\vz}\mathbb{E}_{\veps\sim\gN(0,\rmI)}[R(\vz+\delta\veps)] \approx \nabla_{\vz}R(\vz),
\end{equation}
which can be estimated from $N$ particles as
\begin{equation}
\label{eq:local_zo_estimator}
    \hat{\vg}(\vz)
    = \frac{1}{N}
    \sum_{n=1}^{N}\left[ \frac{1}{\delta}\left(R(\vz+\delta\veps^{(n)})-\bar R\right)\veps^{(n)} \right]
\end{equation}
where $\bar R$ is an optional control variate to reduce variance, typically set as the mean of $R$.
This is the familiar Gaussian-search estimator used in evolution strategies such as NES~\citep{wierstra2014natural}.
Therefore, the most direct gradient-free extension of~\eqref{eq:grad_langevin} is to replace $\nabla_{\vz}R(\vz)$ by $\hat{\vg}(\vz)$.
We refer to this construction as \emph{naive zeroth-order(ZO) Langevin}.

However, while this naive alternative seems theoretically sound, it performs poorly when applied to practical initial-noise optimization. We analyze the reasons for this failure through the following three failure modes.

\textbf{Failure mode I: Nature of Black-box rewards.}
Many black-box rewards can be discrete, quantized, thresholded, ranking-based, or simply locally constant. In this case, the gradient (including its estimator) becomes $0$ and does not carry any signal for the optimization. 
Naively increasing $\delta$ in Eq.~\eqref{eq:local_zo_estimator} can search a broader range to recover reward signal, but this makes the estimator more deviate from the ground truth gradient.

\textbf{Failure mode II: Destabilization from few-particle estimates.}
To make the zeroth-order estimator accurate, the required number of particles $N$ has to grow proportionally to the number of dimensions $d$. However, in most high-dimensional noise spaces for recent generative models, a limited computational budget usually admits $N\ll d$, which introduces unavoidable inaccuracy in the estimator.

Empirically, we observe that this inaccuracy goes beyond making the optimization process less efficient, and frequently drives $\vz$ completely out of distribution, leading to sample collapse. Figure~\ref{fig:failure} shows an example of reward-tilted Langevin dynamics with Pickscore~\citep{kirstain2023pickapic} failing when the actual reward gradient is replaced by $\hat{\vg}(\vz)$. Note that when this happens, the norm of $\vz$ drifts away from the expected value for a Gaussian distribution, rapidly degrading the model output.

\textbf{Failure mode III: Hazard of bad initial states.}
Even if the aforementioned issues are resolved, a sequential and incremental motion of Langevin dynamics suffers from moving between separated modes slowly, as it requires traversing intermediate low-likelihood regions. This makes the finite-budget trajectories sensitive to their initialization.

\begin{figure}
  \begin{minipage}[t]{.34\linewidth}
    \centering
    \includegraphics[width=\linewidth]{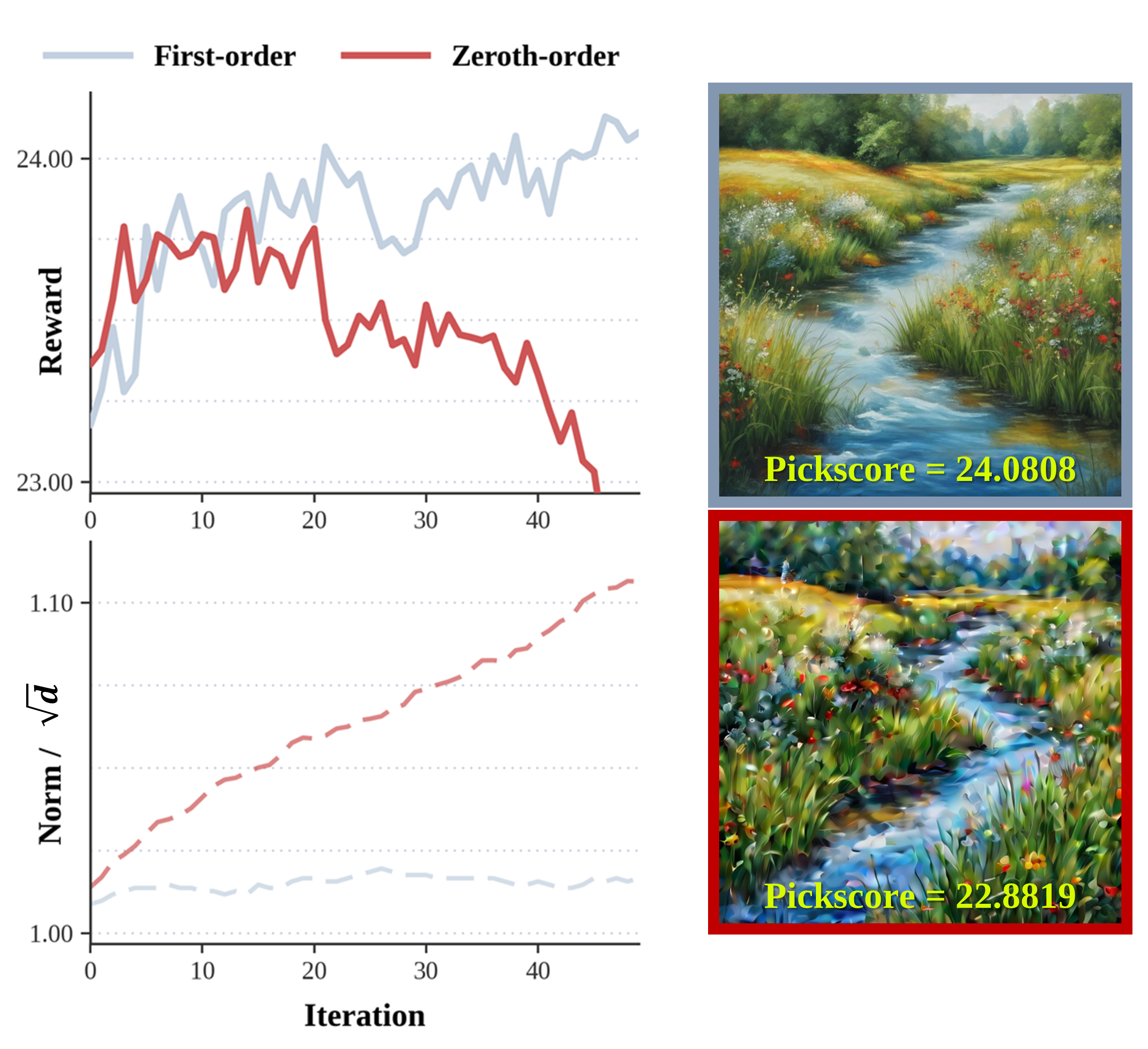}
    \captionof{figure}{Replacing real gradients in Langevin dynamics(grey) with zeroth-order estimates(red).}\label{fig:failure}%
  \end{minipage}\hfill
  \begin{minipage}[t]{.63\linewidth}
    \centering
    \includegraphics[width=\linewidth]{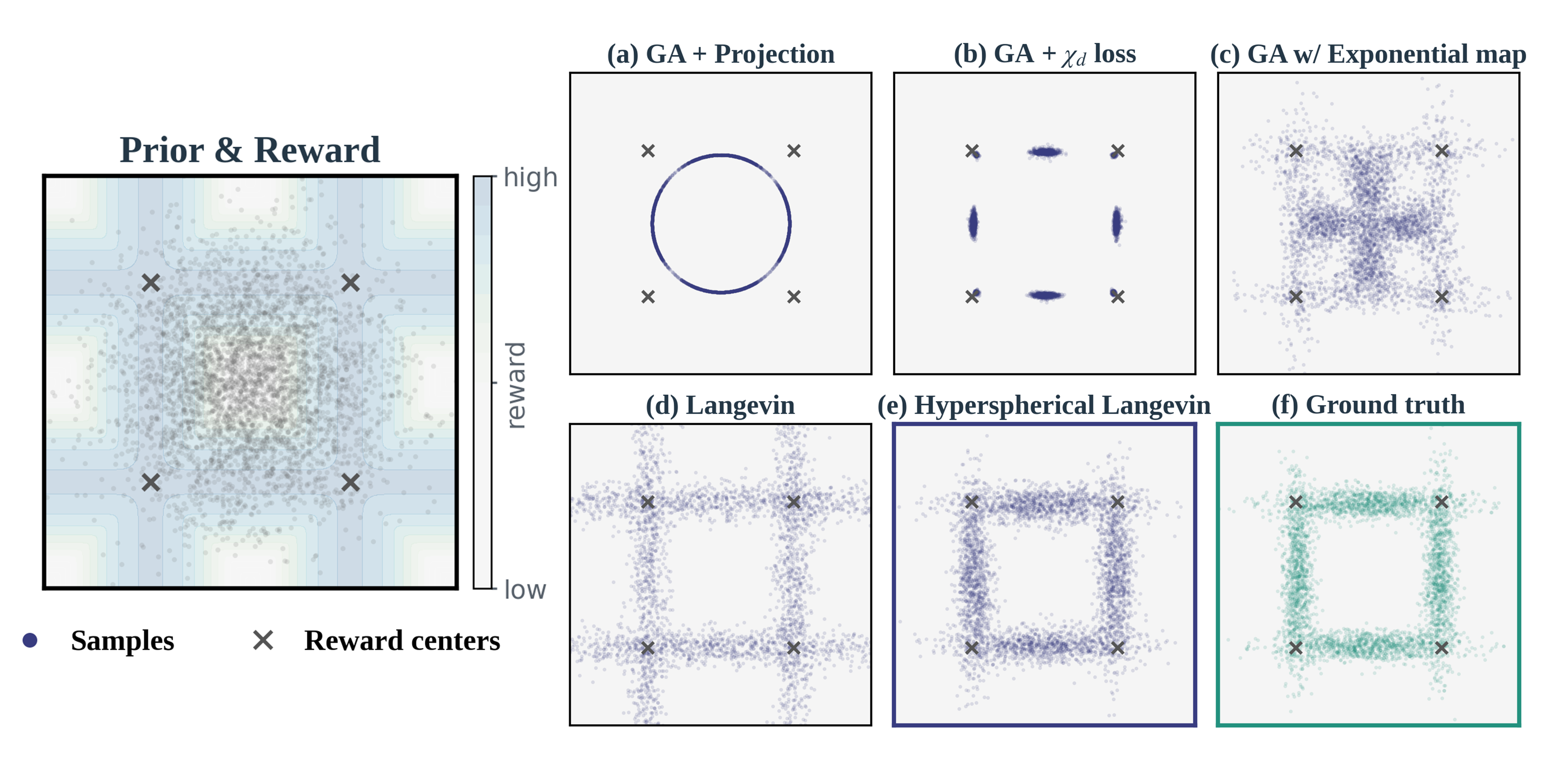}
    \captionof{figure}{2D toy experiment for sampling from reward-tilted target with zeroth-order gradient estimation.} \label{fig:toy}%
  \end{minipage}
  \vspace{-5mm}
\end{figure}

\subsection{ZeNOVA: Noise Optimization for Black-box Rewards with Value Annealing}

\textbf{Zeroth-order value-gradient estimator.}
Note that when the reward gradient control term is removed from Eq.~\eqref{eq:grad_langevin}, the remainder is an Ornstein-Uhlenbeck (OU) process~\citep{oksendal2003stochastic} that preserves $p_0=\gN(0,\rmI)$.
We take this process as a reference and consider a soft value that aggregates future terminal rewards. The OU process admits an exact Gaussian transition over any time interval $h>0$: setting $\sigma^2=1-e^{-2h}$, we can express its transition using a single kernel,
\begin{equation}
\label{eq:ou_transition}
    q_\sigma(\vz'\mid\vz)
    = \gN\!\left(\vz';\sqrt{1-\sigma^2}\vz,\sigma^2\rmI\right),
    \qquad 0<\sigma<1.
\end{equation}
The parameter $\sigma$ specifies the~\emph{horizon} of the OU process: large $\sigma$ nearly forgets the current state and broadly explores the reference prior, whereas small $\sigma$ concentrates around the current state.

Following the previous work in the soft RL~\citep{uehara2024fine} and PIC~\citep{pathintegral}, we define the soft value as the log-expectation of the terminal reward $\exp(R/\lambda)$ under this transition kernel:
\begin{equation}
\label{eq:soft_value}
    V_\sigma(\vz)
    := \lambda\log
    \mathbb{E}_{\vz'\sim q_\sigma(\cdot\mid\vz)}
    \left[\exp\left(\frac{R(\vz')}{\lambda}\right)\right]
\end{equation}
Then, we propose to sample from an intermediate distribution obtained by exponentially tilting the prior with this soft value, instead of the raw reward $R$:
\begin{equation}
\label{eq:intermediate_target}
    p_\sigma^\star(\vz) \propto p_0(\vz) \exp\left(\frac{V_\sigma(\vz)}{\lambda}\right),
\end{equation}
This choice is motivated by two properties of the value function: First, even when the reward function is poorly behaved (\emph{e.g.}, discrete or piecewise constant), the value function is defined as its expectation over a Gaussian kernel, and therefore continuous and well-behaved~\citep{jensen2026value}.
Moreover, this intermediate distribution induced by the soft-value function has an exact relationship to 
$p^\star$:
\begin{restatable}[Interpolation of target distribution]{proposition}{OUInterpolation}
\label{prop:ou_interpolation}
Let reward-tilted distribution $p^\star$ be defined as Eq.~\eqref{eq:target_tilted} and soft-value-tilted distribution $p_\sigma^\star$ as Eq.~\eqref{eq:intermediate_target}. Then $p_\sigma^\star$ is equal to $p^\star$ passed through a Gaussian kernel of Eq.~\eqref{eq:ou_transition}:
\begin{equation}
\label{eq:ou_interpolation}
    p_\sigma^\star(\vz) \propto \int q_\sigma(\vz\mid\vz')\,p^\star(\vz')\,d\vz'.
\end{equation}
\end{restatable}
This implies that $p_\sigma^\star\to p_0$ as $\sigma\to 1$, and $p_\sigma^\star \to p^\star$ as $\sigma\to 0$, where the proof is in Appendix~\ref{proof:prop1}.
Motivated by this interpolation, we use a decreasing horizon schedule $1>\sigma_1>\sigma_2>\cdots>\sigma_{K}>0$ with the aim of tracking the sequence $p_{\sigma_1}^\star,p_{\sigma_2}^\star,\ldots,p_{\sigma_{K}}^\star$ in a coarse-to-fine manner; larger horizons aggregate reward information over a broad neighborhood, while smaller horizons bring the target closer to the original reward tilt and support local refinement.

Crucially, $\nabla V_\sigma$ can be estimated without differentiating $R$.
Applying the Gaussian score identity to~\eqref{eq:soft_value}, we obtain
\begin{equation}
\label{eq:value_grad_exact}
    \nabla_{\vz}V_\sigma(\vz)
    = \frac{\lambda\sqrt{1-\sigma^2}}{\sigma}
    \frac{
        \mathbb{E}_{\veps}\left[
        \exp(R(\sqrt{1-\sigma^2}\vz+\sigma\veps)/\lambda)\,\veps
        \right]
    }{
        \mathbb{E}_{\veps}\left[
        \exp(R(\sqrt{1-\sigma^2}\vz+\sigma\veps)/\lambda)
        \right]
    }.
\end{equation}
See Appendix~\ref{app:value_gradient} for the derivation. With $N$ reward evaluations, we use
\begin{align}
\label{eq:value_grad_estimator}
    \widehat{\vg}_\sigma(\vz)
    &:= \frac{\lambda\sqrt{1-\sigma^2}}{\sigma}
    \sum_{n=1}^{N} \frac{\exp(R(\sqrt{1-\sigma^2}\vz+\sigma\veps^{(n)})/\lambda)}
    {\sum_{j=1}^{N}\exp(R(\sqrt{1-\sigma^2}\vz+\sigma\veps^{(j)})/\lambda)}\veps^{(n)}\qquad \veps^{(n)}\sim\gN(0,\rmI).
\end{align}
The normalized exponential weights are implemented with the usual log-sum-exp stabilization.
For a fixed $\sigma$, we can replace the reward gradient in Eq.~\eqref{eq:grad_langevin} with Eq.~\eqref{eq:value_grad_estimator}.

\textbf{Hyperspherical Langevin dynamics.}
Even for soft value gradient functions, their Monte Carlo estimate can still be noisy for small $N$.
Meanwhile, the standard Gaussian prior is rotationally symmetric, and its density depends only on the radius. Moreover, in high dimensions, Gaussian samples exhibit a uniform likelihood for any direction, whereas their norm concentrates within a narrow range around $\sqrt{d}$.
Indeed, as described in Figure~\ref{fig:failure}, failure of the update with zeroth-order estimation is often accompanied by a drift of the noise norm away from this Gaussian typical range. 

This observation suggests treating the two degrees of freedom separately and formulating the Langevin dynamics in hyperspherical coordinates, separating the prior-sensitive radial motion from the angular motion. Note that sampling $\vz$ from the standard Gaussian can be decomposed as:
\begin{equation}
\vz=\rho \vu\quad\text{where}\quad\rho\sim\chi_d, \vu\in\text{Unif}(\mathbb{S}^{d-1}),
\end{equation}

Therefore, for a given current state $\vz_k$, we decompose it into its norm $\rho_k$ and its unit direction vector $\vu_k$, and separately perform \emph{angular} Langevin update and \emph{radial} Langevin update.

\textit{Angular Langevin on the sphere.}
Let $\rmP_{\vu}^{\perp}:=\rmI-\vu\vu^\top$ denote the projection onto the tangent space of $\vu$.
For $\vg_\sigma(\vz):=\nabla_{\vz}V_\sigma(\vz)$, the angular update uses a gradient and noise on a tangent space, followed by spherical exponential map~\citep{lee2018introduction}:
\begin{align}
\label{eq:spherical_update}
    \vv_k
    &= \rmP_{\vu_k}^{\perp}\left(
    \frac{\eta_{A}\rho_k}{\lambda}
    \widehat{\vg}_{\sigma_k}(\vz_k)
    + \sqrt{2\eta_{A}}\,
    \bm{\xi}_k\right),\qquad \bm{\xi}_k\sim\gN(0,\rmI)\\
    \vu_{k+1} &= \operatorname{Exp}_{\vu_k}(\vv_k)
\end{align}
where $\operatorname{Exp}(\cdot)$ denotes an exponential mapping as follows:
\begin{equation}
\label{eq:sphere_expmap}
    \operatorname{Exp}_{\vu}(\vv)
    = \cos(\|\vv\|_2)\vu
    + \sin(\|\vv\|_2)\frac{\vv}{\|\vv\|_2}
\end{equation}

\textit{Radial Langevin.}
Since the reference radial distribution is $\rho\sim\chi_d$, a radial update step is
\begin{equation}
\label{eq:radial_update}
    \rho_{k+1}
    = \rho_k
    + \eta_{R}
    \left(
        \frac{d-1}{\rho_k}-\rho_k
        +\frac{1}{\lambda}\vu_k^\top\widehat{\vg}_{\sigma_k}(\vz_k)
    \right)
    + \sqrt{2\eta_{R}}\,\xi_k,
    \qquad \xi_k\sim\gN(0,1).
\end{equation}
The radius remains stochastic and reward-adaptive, while noisy high-dimensional guidance is regularized by an analytic one-dimensional prior score rather than by an explicit norm constraint.

For any fixed $\sigma$, this continuous spherical-radial Langevin dynamics has~\eqref{eq:intermediate_target} as its stationary distribution of $\rho\vu$.
This spherical-radial Langevin update allows us to design their dynamics separately(\emph{e.g.}, small step size for radial update), which stabilizes noisy zeroth-order guidance without freezing the radius or relying on an explicit shell projection.

\textbf{Metropolis-Hastings Jumping.}
Since incremental updates may be a somewhat inevitable limitation of Langevin dynamics to simulate an SDE, we briefly examine a fundamentally different alternative of utilizing the Metropolis–Hastings(MH) algorithm:

\begin{restatable}[Metropolis-Hastings algorithm for reward-tilted distribtuion]{proposition}{MHJumping}
\label{prop:mh-jump}
The Metropolis-Hastings chain with proposal $\vz'\sim p_0$ and acceptance probability $\alpha(\vz,\vz')=\min\{1,\exp((R(\vz')-R(\vz))/\lambda)\}$ converges to the reward-tilted distribution $p^\star(\vz)\propto p_0(\vz)e^{R(\vz)/\lambda}$.
\end{restatable}
See Appendix~\ref{proof:prop2} for the derivation. MH chain targets the same distribution with Langevin dynamics and provides global exploration, but it is inefficient as a stand-alone sampler: Namely, for a finite number of proposals, the MH algorithm is strictly worse than Best-of-K(BoK) sampling. 
Moreover, as the sample reward increases, every proposal is drawn afresh regardless of the current high-reward state, and the MH algorithm rejects almost all proposals.
Nonetheless, this has a complementary behavior against Langevin dynamics, allowing occasional global moves to new reward modes with much smaller computation than gradient estimation.
We therefore interleave a small number of independent MH transitions with Langevin updates, dubbed as \emph{MH jumping}. MH jumping is effective in the early stage of the optimization, after which the annealed soft-value and hyperspherical Langevin updates perform coarse-to-fine refinement.

\textbf{ZeNOVA: The Algorithm.} 
Combining the three components yields \textbf{ZeNOVA}(\textbf{Ze}roth-order \textbf{N}oise \textbf{O}ptimization with \textbf{V}alue \textbf{A}nnealing),
where the complete procedure is summarized in Algorithm~\ref{alg:zenova} of Appendix~\ref{sec:algorithm}.
The two principal hyperparameters governing inference-time computation are the number of particles $N$ and the number of sampling iteration $K$.
Increasing $N$ improves the Monte Carlo estimate of the soft value gradient, while increasing $K$ allows the sampler to traverse the interpolation more gradually.
Importantly, ZeNOVA requires only forward evaluations of the generator and reward model.

\section{Experiments}
We evaluate ZeNOVA across a range of deterministic generators and reward functions to demonstrate its effectiveness, scalability, and broad applicability.

\textbf{Toy Experiment.}
We design a reward-aligned sampling problem of 2-D data to identify the effect of existing sampling algorithms under zeroth-order gradient estimation. Starting from $p_0=\gN(0,\rmI)$, we use a synthetic reward with narrow high-reward bands and aim to sample from $p^\star(\vz)\propto p_0(\vz)e^{R(\vz)}$. 
Figure~\ref{fig:toy} compares (a) Gradient Ascent(GA) with fixed-radius projection~\citep{zeno}, GA with radial $\chi_d$ regularization loss~\citep{eyring2024reno}, GA with exponential-map updates~\citep{mono}, (d) Cartesian Langevin updates, and (e) our hyperspherical Langevin updates against (f) ground-truth samples obtained by reward-weighted resampling of prior samples. 
Unlike the other baselines (a)-(d) whose samples concentrate in partial high-reward regions due to their empirical regularizers, our hyperspherical update faithfully reproduces the ground-truth distribution, without artifacts such as samples far along the reward bands in the naive cartesian update.

\begin{figure}[t!]
     \centering
     \includegraphics[width=\linewidth]{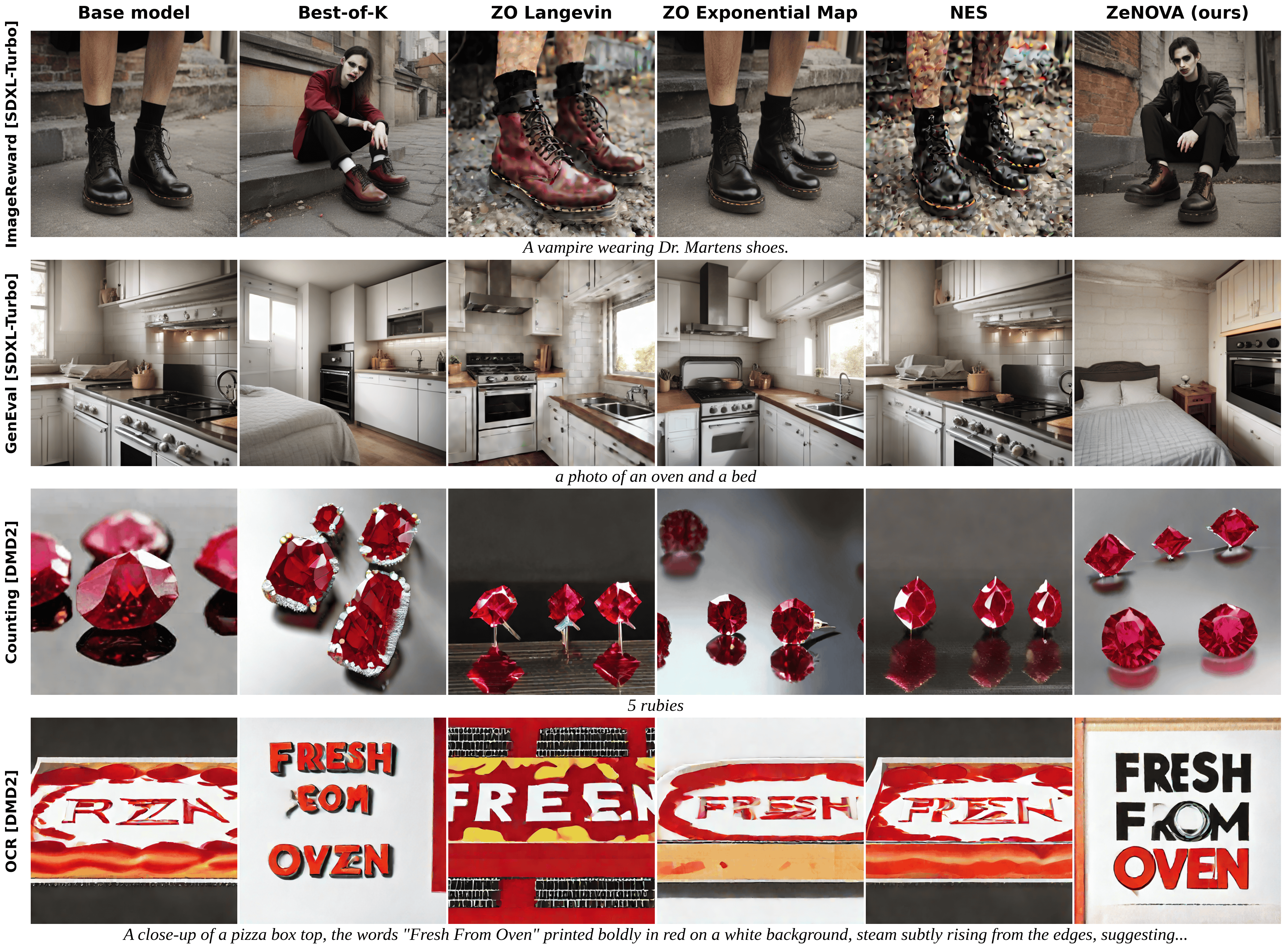}
     \caption{Qualitative image results from SDXL-Turbo and DMD2.}
     \label{fig:qual}
     \vspace{-3mm}
 \end{figure}

\begin{table}[t]
\centering
\caption{Initial noise reward alignment performance on text-to-image models. \textbf{Bold}: best performance, ``-": Not available. }
\label{tab:main_results}
\small
\setlength{\tabcolsep}{3pt}
\resizebox{\linewidth}{!}{%
\begin{tabular}{l|cccccccccccc}
\toprule
& \multicolumn{3}{c}{ImageReward} & \multicolumn{3}{c}{PickScore} & \multicolumn{2}{c}{GenEval} & \multicolumn{2}{c}{Counting} & \multicolumn{2}{c}{OCR} \\
\cmidrule(lr){2-4}\cmidrule(lr){5-7}\cmidrule(lr){8-9}\cmidrule(lr){10-11}\cmidrule(lr){12-13}
Method & Reward $\uparrow$ & Aes. $\uparrow$ & Val. $\uparrow$ & Reward $\uparrow$ & Aes. $\uparrow$ & Val. $\uparrow$ & Reward $\uparrow$ & Aes. $\uparrow$ & Reward $\downarrow$ & Aes. $\uparrow$ & Reward $\uparrow$ & Aes. $\uparrow$ \\
\midrule
\multicolumn{13}{c}{\textbf{SDXL-Turbo}} \\
\midrule
base model & 1.0505 & 6.0282 & 22.8809 & 21.5378 & 5.4020 & 0.5321 & 0.5375 & 5.3853 & 28.375 & 5.6965 & 0.1368 & 5.4256 \\
Best-of-$K$ & 1.6932 & 6.1148 & 23.2829 & 23.1925 & 5.5747 & 0.8949 & 0.7232 & 5.3717 & 2.9 & 5.7096 & 0.4784 & 5.3303 \\
ReNO & 1.7378 & 6.0395 & 22.9585 & \textbf{25.1031} & 5.5721 & 0.7755 & - & - & - & - & - & - \\
ORIGEN & \textbf{1.8198} & 5.1640 & 21.1789 & 24.2321 & 5.5748 & 0.9286 & - & - & - & - & - & - \\
MoNO & 1.6781 & 6.0558 & 23.0221 & 24.6250 & 5.5579 & 0.7675& - & - & - & - & - & - \\
\midrule
ZO Langevin & 1.3968 & 5.8129 & 22.8452 & 22.2734 & 5.5092 & 0.7402 & 0.6875 & 5.4123 & 6.025 & 5.7078 & 0.3823 & 5.3439 \\
ZO Exponential Map & 1.6452 & 6.0507 & 23.2201 & 22.6230 & 5.5437 & 0.7880 & 0.6673 & 5.3912 & 3.325 & 5.6777 & 0.4239 & \textbf{5.4677} \\
NES & 1.6354 & 6.0795 & 23.0850 & 22.3546 & 5.5050 & 0.7583 & 0.6613 & 5.4006 & 5.8 & 5.6811 & 0.4195 & 5.3516 \\
ZeNOVA (ours) & 1.8078 & \textbf{6.1279} & \textbf{23.3584} & 23.4564 & \textbf{5.5957} & \textbf{1.0329} & \textbf{0.7332} & \textbf{5.4278} & \textbf{0.55} & \textbf{5.7222} & \textbf{0.5114} & 5.3864 \\
\midrule
\multicolumn{13}{c}{\textbf{DMD2}} \\
\midrule
base model & 0.5709 & 5.4908 & 21.4339 & 20.4482 & 5.3677 & 0.3438 & 0.4547 & 5.1387 & 14.575 & 5.1468 & 0.0921 & \textbf{5.3661} \\
Best-of-$K$ & 1.6456 & 5.6970 & 22.1676 & 22.0510 & 5.5511 & 0.8915 & 0.7224 & 5.1387 & 1.85 & 5.1557 & 0.4646 & 5.0526 \\
ReNO & 1.3789 & 5.4232 & 20.9871 & \textbf{23.4025} & 5.4159 & 0.5440 & - & - & - & - & - & - \\
ORIGEN & 1.2708 & 4.9731 & 19.6959 & 21.8586 & 5.3021 & 0.2489 & - & - & - & - & - & - \\
MoNO & 1.2816 & 5.3627 & 21.0667 &23.1507 & 5.4140 & 0.3623 & - & - & - & - & - & - \\
\midrule
ZO Langevin & 1.4898 & 5.4877 & 21.8468 & 21.8061 & 5.4043 & 0.7811 & 0.6597 & 5.1388 & 9.775 & 4.6922 & 0.3961 & 5.3055 \\
ZO Exponential Map & 1.4100 & 5.5935 & 21.9041 & 21.9611 & 5.5548 & 0.6782 & 0.6383 & 5.1436 & 1.35 & \textbf{5.2610} & 0.2351 & 5.3417 \\
NES & 1.2996 & 5.5763 & 21.7994 & 21.6149 & 5.5007 & 0.6193 & 0.6251 & 5.1531 & 1.975 & 5.0784 & 0.2073 & 5.3637 \\
ZeNOVA (ours) & \textbf{1.7683} & \textbf{5.7210} & \textbf{22.4022} & 22.6398 & \textbf{5.6130} & \textbf{0.9187} & \textbf{0.7275} & \textbf{5.1703} & \textbf{0.075} & 5.1902 & \textbf{0.5237} & 5.1575 \\
\bottomrule
\end{tabular}%
}
\vspace{-3mm}
\end{table}

\textbf{Text-to-Image Generation.}
We evaluate initial noise alignment for pretrained text-to-image generation with SDXL-turbo~\citep{sauer2024adversarial} and DMD2~\citep{yin2024improved} on five tasks: ImageReward~\citep{xu2024imagereward}, PickScore~\citep{kirstain2023pickapic}, GenEval~\citep{ghosh2023geneval}, Query object counting~\citep{qian2025t2icount}, and OCR~\citep{cui2025paddleocr}. The generator parameters remain fixed, and alignment is performed by updating the initial noise. ImageReward and PickScore provide differentiable preference rewards, while GenEval, Counting, and OCR evaluate compositional alignment, object counting, and text rendering through black-box reward evaluations. We additionally report aesthetic score~\citep{schuhmann2022aesthetic} as a separate evaluation metric to assess how task-specific alignment affects visual quality. See Appendix~\ref{subsec:image_setting} for the detailed settings on the reward functions and utilized text prompt datasets.

We include uncontrolled generation and Best-of-K, which selects the best sample among $K$ independently drawn initial noises. The zeroth-order baselines comprise 
ZO Langevin, exponential-map updates with a ZO gradient estimator, and Natural Evolution Strategies (NES)~\citep{wierstra2014natural}. For differentiable rewards, we also include ReNO~\citep{eyring2024reno}, ORIGEN~\citep{min2026origen}, and MoNO~\citep{mono} as baselines that utilize backpropagation through the generator and reward model. 

Table~\ref{tab:main_results} shows that ZeNOVA improves over uncontrolled generation, Best-of-$K$, and every reported zeroth-order optimization baseline on all five task metrics for both generative models. 
Specifically, in black-box reward tasks,  ZeNOVA achieves performance gains of $+0.046$ on GenEval and $+0.088$ on OCR from the best zeroth-order baseline, and reduces MAE of counting by around $40\%$.
These results show improvements over the alternatives on objectives that do not provide gradients.
Note that on ImageReward and PickScore where the reward gradient is accessible, gradient-based algorithms of ReNO and ORIGEN often attain higher reward than ours. Here, leveraging the fact that ImageReward and PickScore both measure the same aspect of human preference, we evaluated both runs with the reward not employed during optimization as a validation metric. We observe that gradient-based methods are prone to reward hacking, allowing our method to achieve the highest validation metric.

In terms of image quality, ZeNOVA achieved the highest Aesthetic score across various scenarios, while some unstable zeroth-order methods and first-order methods prone to reward hacking exhibited significant drops. While this is partly because some optimized rewards are already associated with image quality, but the results in other cases suggest that ZeNOVA properly preserves the prior distribution of the original generative model's input.


\begin{figure}[t!]
    \centering
    \includegraphics[width=\linewidth]{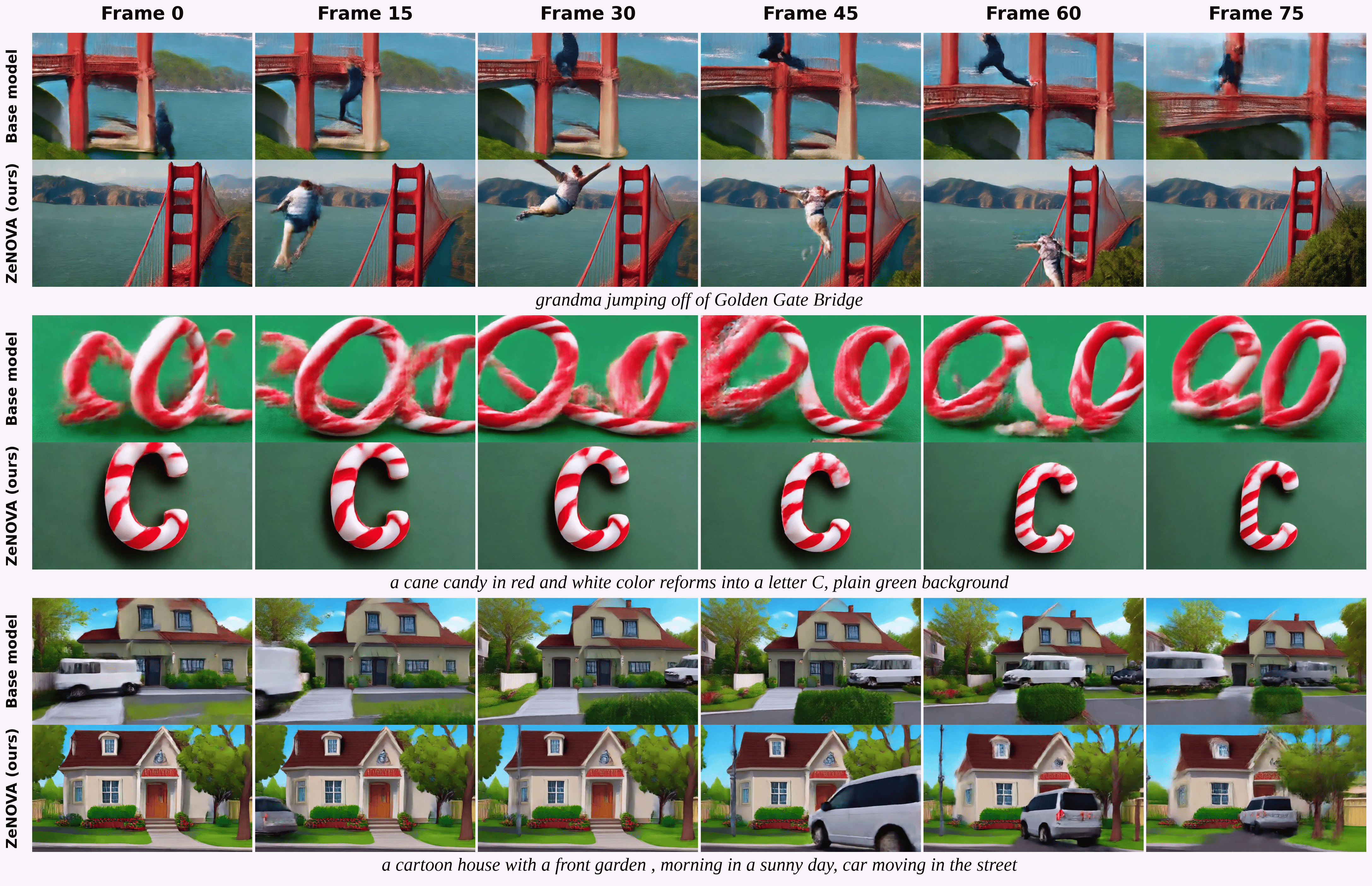}
    \caption{Qualitative video results from rCM with ZeNOVA, optimizing VideoAlign scores.}
    \label{fig:qual_video}
    \vspace{-5mm}
\end{figure}

\textbf{Text-to-Video Generation.} Due to the increased data dimensionality introduced by the additional temporal axis, a growing interest for models that generate video frames or chunks in a few steps was emerged.
We further applied ZeNOVA to a few-step distilled text-to-video flow-matching model rCM~\citep{rcm}. Following~\citet{xue2025dancegrpo}, we employed three VideoAlign~\citep{videoalign} benchmark score for Visual Quality (VQ), Motion Quality (MQ), and Text Alignment (TA). Text prompts from VidProM~\citep{wang2024vidprom} dataset were used. More details can be found in Appendix~\ref{subsec:video_setting}.

\begin{wraptable}{r}{0.54\textwidth}
\vspace{-3mm}
\centering
\caption{Quantitative performance on text-to-video models. \textbf{Bold}: Best performance.}
\label{tab:video}
\small
\resizebox{\linewidth}{!}{%
\begin{tabular}{l|cccc}
\toprule
   & \multicolumn{4}{c}{VideoAlign} \\
\cmidrule(lr){2-5}Method   & VQ & MQ & TA & All \\
\midrule
base model & 3.4854 & 1.3684 & 2.1731 & 6.1997 \\
ZeNOVA & \textbf{6.4464} & \textbf{3.3124} & \textbf{3.9088} & \textbf{10.8146} \\
\bottomrule
\end{tabular}
}
\vspace{-2mm}
\end{wraptable}

Table~\ref{tab:video} and Figure~\ref{fig:qual_video} show that ZeNOVA improves over uncontrolled generation in all four optimization settings. These gains demonstrate the scalability of the proposed initial-noise alignment method for a task with much more dimensionality and complexity, supporting both individual reward objectives and their combined optimization without modifying the pretrained generator.

\section{Discussion}
\textbf{Effect of Particle Count per Iteration.}
Zeroth-order methods require reward evaluations for a fixed number($N$) of particles per iteration to estimate each method's intended update direction of the initial noise; as $N$ increases, this estimate approaches the theoretical ground truth and the performance converges. In practice, since the number of available reward evaluations is finite, the performance behavior with a small number of particles becomes a critical factor in the practicality of the algorithm.

\begin{figure}
  \begin{minipage}[t]{.49\linewidth}
    \centering
    \includegraphics[width=\linewidth]{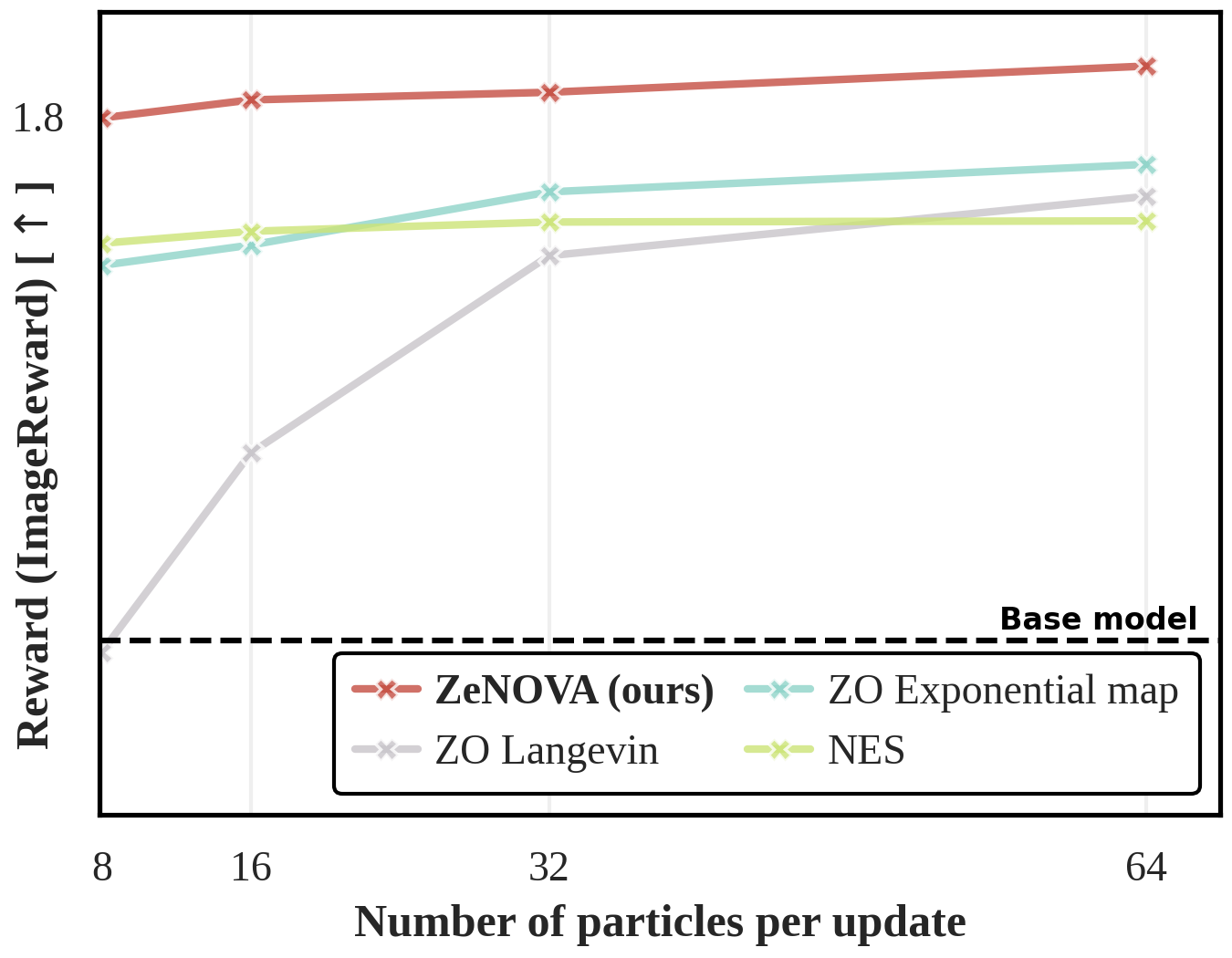}
    \captionof{figure}{Effect of particle count per iteration for zeroth-order methods.}\label{fig:scaling}%
  \end{minipage}\hfill
  \begin{minipage}[t]{.49\linewidth}
    \centering
    \includegraphics[width=\linewidth]{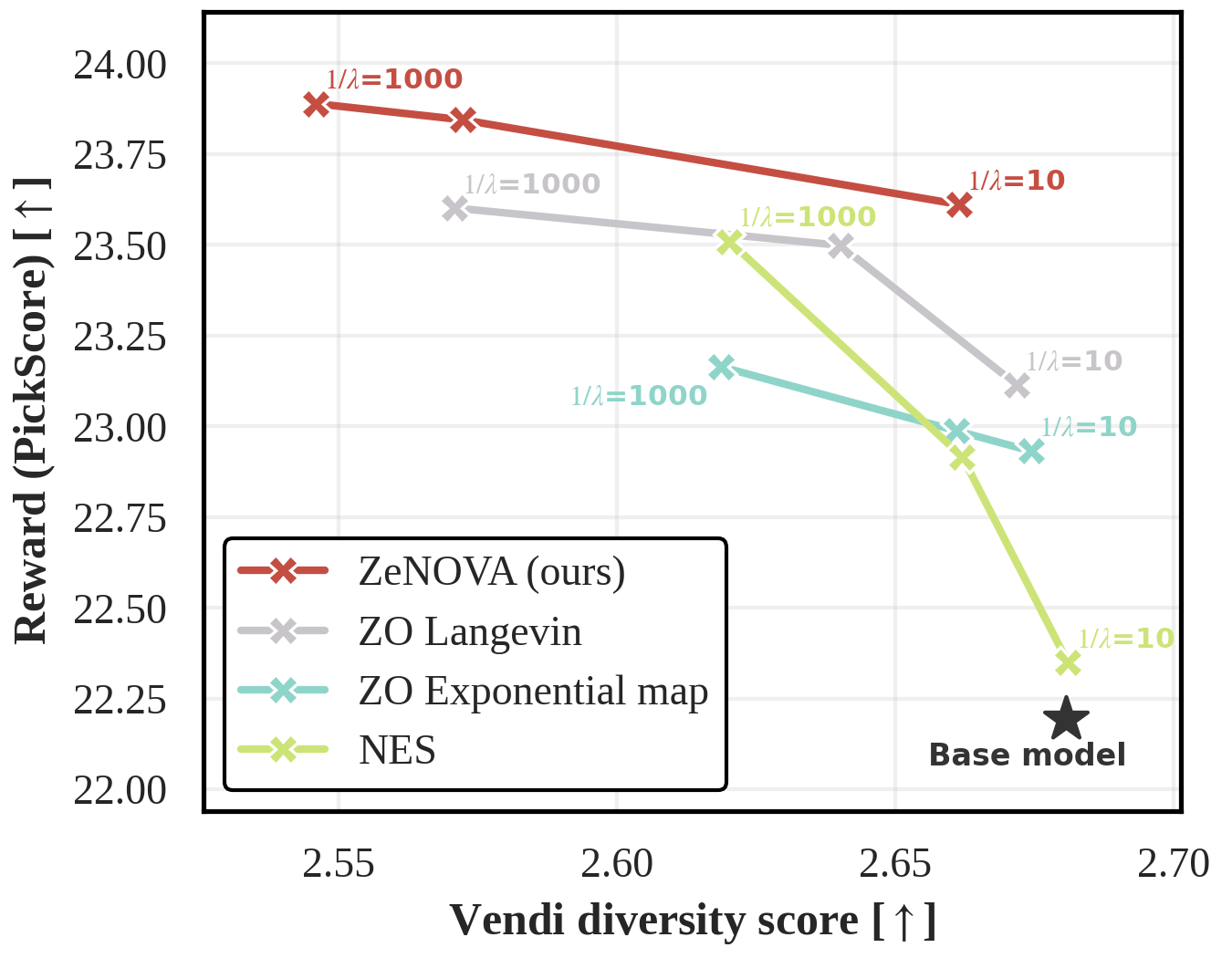}
    \captionof{figure}{Pareto curve between PickScore and Vendi diversity at different guidance scales.}\label{fig:diversity}%
  \end{minipage}
  \vspace{-5mm}
\end{figure}

Figure~\ref{fig:scaling} compares ZeNOVA with Cartesian Langevin, exponential-map updates, and NES with ImageReward on HPDv2 prompts, holding the number of iterations fixed at $K=25$. Reward increases with $N$ for every evaluated method, while ZeNOVA achieves the highest reward at every reported particle count and exhibits a small degradation as the particle budget decreases; reducing $N$ from $64$ to $8$ lowers ZeNOVA's ImageReward by only $0.0743$. Notably, ZeNOVA with only $8$ particles outperforms all three baselines with $64$ particles. These results demonstrate that ZeNOVA preserves strong alignment under limited reward-query budgets and are consistent with the stabilizing role of hyperspherical updates under noisy zeroth-order guidance.

\textbf{Reward-Diversity Trade-off.}
ZeNOVA and the baselines have a guidance scale hyperparameter $\lambda$ to control the trade-off between prior preservation and reward.
Figure~\ref{fig:diversity} plots the Pareto curve between PickScore reward and Vendi diversity~\citep{friedman2023vendi} of 64 images per prompt, for different methods at three guidance scales. 
For all methods, increasing $1/\lambda$ improves the average reward while diversity decreases, indicating that stronger guidance yields reward gains by making the output distribution focus on a narrow high-reward region. Meanwhile, our approach achieves a dominant Pareto front located at the most upright of the plot, indicating higher diversity in its output for a fixed level of reward alignment.

\begin{wraptable}{r}{0.58\textwidth}
\vspace{-3mm}
\centering
\caption{Ablation study of the components of ZeNOVA.}
\label{tab:discussion_ablation}
\small
\resizebox{\linewidth}{!}{%
\begin{tabular}{lccc}
\toprule
  Variants& ImageReward & PickScore & OCR \\
\midrule
 ZeNOVA& \textbf{1.8078} & \textbf{23.4564} & \textbf{0.5114} \\
 - MH jumping & 1.7262 & 23.1867 & 0.4412 \\
 $\quad$- Value Annealing& 1.6984 & 23.2960 & 0.3756 \\
 $\qquad$- Hyperspherical update& 1.3968 & 22.2734 & 0.3823 \\
\bottomrule
\end{tabular}
}
\vspace{-2mm}
\end{wraptable}

\textbf{Ablation Study.}
We've conducted an ablation study on more detailed design choices of the proposed ZeNOVA in Table~\ref{tab:discussion_ablation} to analyze and emphasize their role. 
It is shown that all three components contributed to the method's performance and stability. 
We note that introducing the annealed soft value function yielded marginal or no performance improvements on ImageReward and PickScore, but led to a substantial improvement on OCR reward; the reward landscapes of the former two rewards are continuous with many high-reward modes, allowing sufficiently high rewards to be achieved using only local reward gradient information around the initial point. In the more challenging black-box reward of OCR, the value function allows the method to access a more global and meaningful reward gradient signal.

\vspace{-2mm}
\section{Conclusion}
We introduced ZeNOVA, a stable and efficient initial-noise optimization framework for generative models through initial-noise optimization. Guided by a reward-tilted sampling objective, ZeNOVA combines annealed soft-value guidance to extract useful signals from irregular rewards, where several crucial challenges are adequately resolved to enable stability and global exploration.
Experimental results on various scenarios demonstrate that ZeNOVA consistently achieves higher rewards than the zeroth-order baselines and retains strong performance with less computational cost.

\subsection*{AI use statement}

In this work, we used generative AI tools to [assist with translation]. 
We have not used generative AI tools to [help develop theoretical models or conceptual frameworks, formulate mathematical claims, provide critical ingredients for proving mathematical claims, assist in the writing of proofs, propose or refine hypotheses, design or provide feedback on research methodology or experiments, implement methods, clean and reformat dataset, interpret results], 
and [generate synthetic data sets, support qualitative and thematic data analysis] are not applicable to this work. 
Additionally, we used generative AI tools to [create or edit software code, edit a research paper to improve readability]. We have reviewed all AI-assisted work. [All authors reviewed and refined LLM-generated drafts and code snippets before using them]. We take responsibility for the final content of this work, including text, claims or artifacts produced with the aid of generative AI.


\subsection*{Ethics statement}


This paper presents work on image generative models, and we acknowledge that generative models can produce misleading, biased, or not-safe-for-work visual content, which may amplify misinformation, prejudice, or privacy violations.

\subsection*{Reproducibility statement}

For theoretical results, clear explanations of any assumptions and a complete proof of the claims are included in the Appendix.
Any datasets used in the experiments are open access; a complete description of the data preparation are provided in the Appendix.



\bibliography{iclr2027_conference}
\bibliographystyle{iclr2027_conference}

\clearpage
\appendix
\section{Theoretical Results}

\subsection{Proof of Proposition 1} \label{proof:prop1}
\OUInterpolation*
\begin{proof} 
Starting from the defition of the soft value function,
\begin{equation}
\begin{aligned}
    V_\sigma(\vz) &:=\lambda\log\mathbb{E}_{\vz'\sim q_\sigma(\cdot\mid\vz)}
    \left[e^{R(\vz')/\lambda}\right] \\
    \Leftrightarrow\exp\left(\frac{V_\sigma(\vz)}{\lambda}\right)
    &=\exp\!\left(\log\mathbb{E}_{\vz'\sim q_\sigma(\cdot\mid\vz)}
    \left[e^{R(\vz')/\lambda}\right]\right)\\
    &=\mathbb{E}_{\vz'\sim q_\sigma(\cdot\mid\vz)}
    \left[e^{R(\vz')/\lambda}\right]\\
    &=\int q_\sigma(\vz'\mid\vz)e^{R(\vz')/\lambda}\,d\vz'
\end{aligned}
\end{equation}
Substituting this into~\eqref{eq:intermediate_target} gives
\begin{equation}
\begin{aligned}
    p_\sigma^\star(\vz) &\propto p_0(\vz) \exp\left(\frac{V_\sigma(\vz)}{\lambda}\right) \\
    &= \int p_0(\vz)q_\sigma(\vz'\mid\vz)e^{R(\vz')/\lambda}\,d\vz'
\end{aligned}
\end{equation}

Here, since OU kernel is reversible with respect to $p_0$, we have $p_0(\vz)q_\sigma(\vz'\mid\vz)=p_0(\vz')q_\sigma(\vz\mid\vz')$. Thus,
\begin{equation}
\begin{aligned}
    p_\sigma^\star(\vz) &\propto \int p_0(\vz')q_\sigma(\vz\mid\vz')e^{R(\vz')/\lambda}\,d\vz'\\
    &\propto \int q_\sigma(\vz\mid\vz')p^\star(\vz')\,d\vz'.
\end{aligned}
\end{equation}
\end{proof}

\subsection{Proof of Proposition 2} \label{proof:prop2}
\MHJumping*
\begin{proof}
    To construct a Markov chain for sampling from an unnormalized density $\widetilde p(\vz)$, Metropolis-Hastings algorithm draws a candidate $\vz'\sim q(\cdot\mid\vz)$ and accepts it with probability
\begin{equation}
    \alpha_{\mathrm{MH}}(\vz,\vz')
    =\min\!\left\{1,
    \frac{\widetilde p(\vz')q(\vz\mid\vz')}
    {\widetilde p(\vz)q(\vz'\mid\vz)}\right\},
\end{equation}
retaining the current state if the proposal is rejected.
\begin{equation}
    p^\star(\vz)\propto p_0(\vz)\exp\!\left(\frac{R(\vz)}{\lambda}\right),
\end{equation}
To apply Metropolis-Hastings to an unnormalized reward-tilted distribution $\widetilde p(\vz)=p_0(\vz)e^{R(\vz)/\lambda}$, on can choose the independent Gaussian prior proposal and reward-based acceptance probability as
\begin{equation}
    q(\vz'\mid\vz)=p_0(\vz')=\gN(\vz';0,\rmI).
\end{equation}

By this selection, Metropolis-Hastings acceptance ratio for $\alpha(\vz,\vz')$ becomes
\begin{equation}
\begin{aligned}
    \frac{\widetilde p(\vz')q(\vz\mid\vz')}
    {\widetilde p(\vz)q(\vz'\mid\vz)}
    =\frac{p_0(\vz')e^{R(\vz')/\lambda}p_0(\vz)}
    {p_0(\vz)e^{R(\vz)/\lambda}p_0(\vz')}
    =\exp\!\left(\frac{R(\vz')-R(\vz)}{\lambda}\right),
\end{aligned}
\end{equation}
which reconstructs the claimed algorithm to generate samples from $p^\star(\vz)\propto p_0(\vz)e^{R(\vz)/\lambda}$.
\end{proof}

\subsection{Derivation of Zeroth-order Value Gradient Estimator}
\label{app:value_gradient}

Stochastic Optimal Control theory describes the optimal control term injected into a given reference process to maximize an entropy-regularized objective and relates it to the soft value function. 
Specifically, in our OU reference process over a timestep $t=0$ to  $t=h$,
\begin{equation}
    d\vz_s=-\vz_s\,ds+\sqrt{2}\,d\vw_s,
    \qquad \vz_0=\vz.
\end{equation}
the optimal control $\vu^\star(\vz,0)$ is defined to maximize the following objectives under the controlled process.

\begin{align}
      &\max\mathbb{E}_{u}\left[R(\vz_h^u) -\frac{\lambda}{2}\int_0^h\|\vu_s\|^2ds\middle|\,\vz_0^u=\vz\right] \\
    &\text{s.t. } d\vz_s^u=-\vz_s^uds + \sqrt{2}\bigl(\vu_s\,ds+d\vw_s\bigr), \vz_0^u=\vz
\end{align}

With $\sigma^2:=1-e^{-2h}$, Hamilton-Jacobi-Bellman equation~\citep{textbook_soc} states that  the optimal control at $t=0$ $\vu^\star(\vz,0)$ satisfies
\begin{equation}
    \vu^\star(\vz,0)
    =\frac{\sqrt{2}}{\lambda}\nabla_{\vz}V_\sigma(\vz).
\end{equation}
Furthermore, the path-integral control identity~\citep{pathintegral} gives
\begin{equation} \label{eq:pic_value_gradient}
    \vu^\star(\vz,0) = \lim_{dt\to0}\frac{\mathbb{E}_{\text{ref}}[e^{R(\vz_h)/\lambda}d\vw_0\mid\vz_0=\vz]}{dt\mathbb{E}_{\text{ref}}[e^{R(\vz_h)/\lambda}\mid\vz_0=\vz]}.
\end{equation}
Here, $d\vw_0$ and $dt$ denote the injected Brownian motion of a single trajectory at $t=0$ and its corresponding timestep increment.
All expectations in this identity are under the uncontrolled reference process. The terminal reward is evaluated at $\vz_h$.

Now, we reformulate Eq.~\eqref{eq:pic_value_gradient} into terms that are more easily tractable. First, we can get the exact OU solution for $\vz_h$ as
\begin{equation}
    \vz_h=\sqrt{1-\sigma^2}\vz+\sigma\veps,
    \qquad \veps\sim\gN(0,\rmI).
\end{equation}

Our goal is to rewrite the numerator of Eq.~\eqref{eq:pic_value_gradient}, which contains both the terminal reward and the initial Brownian increment $d\vw_0$, using only the terminal Gaussian sample $\vz_h$. 
Here, note that the reward weight depends only on $\vz_h$. 
Therefore, applying the tower property to the numerator of Eq.~\eqref{eq:pic_value_gradient} yields
\begin{align}
    \mathbb{E}[e^{R(\vz_h)/\lambda}d\vw_0]
    &=\mathbb{E}\!\left[\mathbb{E}[e^{R(\vz_h)/\lambda}d\vw_0\mid\vz_h]\right]\\
    &=\mathbb{E}\!\left[e^{R(\vz_h)/\lambda}\mathbb{E}[d\vw_0\mid\vz_h]\right]
\end{align}

For $\mathbb{E}[d\vw_0\mid\vz_h]$, we use the conditional mean formula jointly Gaussian variables:
\begin{equation}
    \mathbb{E}[X\mid Y]
    =\mathbb{E}[X]+\operatorname{Cov}(X,Y)
    \operatorname{Var}(Y)^{-1}(Y-\mathbb{E}[Y]).
\end{equation}
Applying this formula with $X=d\vw_0$ and $Y=\vz_h$, and using $\mathbb{E}[d\vw_0]=0$ and $\mathbb{E}[\vz_h]=\sqrt{1-\sigma^2}\vz$, gives
\begin{equation}
    \mathbb{E}[d\vw_0\mid\vz_h]
    =0+\operatorname{Cov}(d\vw_0,\vz_h)
      \operatorname{Var}(\vz_h)^{-1}(\vz_h-\sqrt{1-\sigma^2}\vz)
\end{equation}

Here, we have a stochastic-integral representation of the terminal state of the OU process as $\vz_h=e^{-h}\vz+\sqrt{2}\int_0^h e^{-(h-s)}d\vw_s$. Both $d\vw_0$ and $\vz_h$ are thus linear functions of Gaussian noise and are jointly Gaussian. The initial noise contributes to $\vz_h$ through the interval $[0,dt]$, while later increments are independent of it. Their cross-covariance and terminal covariance are consequently
\begin{equation}
\begin{aligned}
    \operatorname{Cov}(d\vw_0,\vz_h)
    &=\sqrt{2}\int_0^{dt}e^{-(h-s)}ds\rmI
      =\sqrt{2}\sqrt{1-\sigma^2}(e^{dt}-1)\rmI,\\
    \operatorname{Var}(\vz_h)&=\sigma^2\rmI.
\end{aligned}
\end{equation}

Now we can complete the formula:

\begin{align}
    \mathbb{E}[d\vw_0\mid\vz_h]&=\frac{\sqrt{2}\sqrt{1-\sigma^2}(e^{dt}-1)}{\sigma^2}
      (\vz_h-\sqrt{1-\sigma^2}\vz) \\
    \mathbb{E}[e^{R(\vz_h)/\lambda}d\vw_0]&=\frac{\sqrt{2}\sqrt{1-\sigma^2}(e^{dt}-1)}{\sigma^2}
      \mathbb{E}\!\left[e^{R(\vz_h)/\lambda}
      (\vz_h-\sqrt{1-\sigma^2}\vz)\right].
\end{align}
Substituting it into Eq.~\eqref{eq:pic_value_gradient} gives
\begin{equation}
\begin{aligned}
    \vu^\star(\vz,0)
    &=\lim_{dt\to0}
      \frac{\sqrt{2}\sqrt{1-\sigma^2}}{\sigma^2}
      \frac{e^{dt}-1}{dt}
      \frac{\mathbb{E}[e^{R(\vz_h)/\lambda}(\vz_h-\sqrt{1-\sigma^2}\vz)]}
      {\mathbb{E}[e^{R(\vz_h)/\lambda}]}\\
    &=\frac{\sqrt{2}\sqrt{1-\sigma^2}}{\sigma^2}
      \frac{\mathbb{E}[e^{R(\vz_h)/\lambda}(\vz_h-\sqrt{1-\sigma^2}\vz)]}
      {\mathbb{E}[e^{R(\vz_h)/\lambda}]}
\end{aligned}
\end{equation}
Finally, we substitute $\vz_h=\sqrt{1-\sigma^2}\vz+\sigma\veps$:
\begin{equation}
    \vu^\star(\vz,0)
    =\frac{\sqrt{2}\sqrt{1-\sigma^2}}{\sigma}
    \frac{\mathbb{E}_{\veps}[e^{R(\sqrt{1-\sigma^2}\vz+\sigma\veps)/\lambda}\veps]}
    {\mathbb{E}_{\veps}[e^{R(\sqrt{1-\sigma^2}\vz+\sigma\veps)/\lambda}]}.
\end{equation}

The resulting expression depends only on a Gaussian perturbation and its terminal reward, so no full OU trajectory or initial Brownian increment needs to be sampled.
Multiplying by $\lambda/\sqrt{2}$ recovers~\eqref{eq:value_grad_exact}.

\subsection{Derivation of Hyperspherical Langevin Dynamics}
We decompose a sample $\vz$ from a standard Gaussian as its norm $\rho$ and its unit direction vector $\vu$.
\begin{equation}
\label{eq:polar_decomp}
    \vz=\rho\vu,
    \qquad
    \rho:=\|\vz\|_2>0,
    \qquad
    \vu\in\mathbb{S}^{d-1}
\end{equation}
When we convert the soft-value-tilted distribtuion with $\rho$ and $\vu$ as well, Eq.~\eqref{eq:intermediate_target} becomes
\begin{equation}
\label{eq:polar_target}
    \widetilde p_\sigma^\star(\rho,\vu)
    \propto
    \rho^{d-1}
    \exp\left(
        -\frac{\rho^2}{2}
        +\frac{V_\sigma(\rho\vu)}{\lambda}
    \right).
\end{equation}
This expression exposes two qualitatively different directions: the angular component moves along the Gaussian shell, whereas the radial component controls how far the sample lies from the origin.
Rather than allowing a noisy $d$-dimensional update to perturb both components implicitly, we construct Langevin dynamics for them explicitly.

\textbf{Angular Langevin on the sphere.}
With a fixed radial $\rho$, the score function of the angular conditional distribution induced by~\eqref{eq:polar_target} is
\begin{equation}
\label{eq:angular_score}
    \nabla_{\mathbb{S}}\log \widetilde p_\sigma^\star(\rho,\vu)
    = \rmP_{\vu}^{\perp}\left(\frac{\rho}{\lambda}
    \vg_\sigma(\rho\vu)\right),
\end{equation}

where $\vg_\sigma(\vz):=\nabla_{\vz}V_\sigma(\vz)$ and
$\rmP_{\vu}^{\perp}:=\rmI-\vu\vu^\top$
denotes the projection onto the tangent space of $\vu$.

Since $\vu$ must lie on a unit sphere, we use intrinsic spherical Langevin dynamics
\begin{equation}
\label{eq:spherical_sde}
    d\vu_s
    = \nabla_{\mathbb{S}}\log \widetilde p_\sigma^\star(\rho_s,\vu_s)\,ds
    + \sqrt{2}\,d\vw_s^{\mathbb{S}},
\end{equation}
where $\vw_s^{\mathbb{S}}$ denotes Brownian motion also lying on the unit sphere.
A practical Euler step implementation of this with step size $\eta_A$ can be a tangent Gaussian followed by the spherical exponential map~\citep{lee2018introduction}:
\begin{align}
    \vv_k
    &= \eta_{A}
    \frac{\rho_k}{\lambda}
    \rmP_{\vu_k}^{\perp}\widehat{\vg}_{\sigma_k}(\vz_k)
    + \sqrt{2\eta_{A}}\,
    \rmP_{\vu_k}^{\perp}\bm{\xi}_k,\\
    \vu_{k+1}
    &= \operatorname{Exp}_{\vu_k}(\vv_k),
    \qquad \bm{\xi}_k\sim\gN(0,\rmI),
\end{align}
where $\text{Exp}_{\vu}(\vv)$ implements the rotation of $\vu$ to the origin toward the direction of $\vv$ as
\begin{equation}
    \operatorname{Exp}_{\vu}(\vv)
    = \cos(\|\vv\|_2)\vu
    + \sin(\|\vv\|_2)\frac{\vv}{\|\vv\|_2}.
\end{equation}
with the continuous extension $\operatorname{Exp}_{\vu}(\vzero)=\vu$.
By construction, both the estimated reward drift and the Langevin noise remain tangent to the sphere.

\textbf{Radial Langevin.}
For soft-value-tilted distribution, with a fixed angular state $\vu$, the score function of the radial conditional induced by~\eqref{eq:polar_target} is

\begin{equation}
\label{eq:radial_score}
    \frac{d}{d\rho}\log\widetilde p_\sigma^\star(\rho\mid\vu)
    = \frac{d-1}{\rho}-\rho
    + \frac{1}{\lambda}\vu^\top\vg_\sigma(\rho\vu).
\end{equation}
Note that the first two terms are equal to the score function of the standard Gaussian's radial distribution $\rho\sim\chi_d$.
We therefore directly apply one-dimensional Langevin dynamics to the radius,
\begin{equation}
\label{eq:radial_sde}
    d\rho_s
    = \left(
        \frac{d-1}{\rho_s}-\rho_s
        +\frac{1}{\lambda}\vu_s^\top\vg_\sigma(\rho_s\vu_s)
    \right)ds
    +\sqrt{2}\,dw_s.
\end{equation}
A practical Euler step with step size $\eta_{R}$ is
\begin{equation}
    \rho_{k+1}
    = \rho_k
    + \eta_{R}
    \left(
        \frac{d-1}{\rho_k}-\rho_k
        +\frac{1}{\lambda}\vu_k^\top\vg_{\sigma_k}(\vz_k)
    \right)
    + \sqrt{2\eta_{R}}\,\xi_k,
    \qquad \xi_k\sim\gN(0,1).
\end{equation}
The first two terms provide the exact Gaussian radial bias, while the scalar projection $\vu^\top\widehat{\vg}_\sigma$ captures the reward-induced radial preference. Thus the radius remains stochastic and reward-adaptive, unlike several empirical regularizations like explicit norm constraints.

\section{Implementation Details}\label{sec:detail}
\subsection{Detailed Algorithm of ZeNOVA} \label{sec:algorithm}
Algorithm~\ref{alg:zenova} describes a detailed algorithm of the proposed ZeNOVA with Annealed Soft Value, Hyperspherical Langevin, and MH Jumping. Note that the reciprocal of the temperature $1/\lambda$ can be considered as an effective reward guidance scale, making the method more focused on reward maximization than prior preservation.

\begin{algorithm}[t]
\caption{ZeNOVA with Annealed Soft Value, Hyperspherical Langevin, and MH Jumping}
\label{alg:zenova}
\begin{algorithmic}[1]
\Require Generator $G_\theta$, reward $r$, temperature $\lambda$, horizons $1>\sigma_1>\cdots>\sigma_{K}>0$, number of particles $N$, number of MH jumping $N_{\text{MH}}$, number of iteration $K$, step sizes $\eta_{A},\eta_{R}$
\State $\vz_0\sim\mathcal{N}(0,\mathbf{I})$; $\rho_0\gets\|\vz_0\|_2$; $\vu_0\gets\vz_0/\rho_0$
\For{$k=0,\ldots,K-1$}
    \For{$j=1,\ldots,N_{\text{MH}}$} \Comment{independent MH jumping}
        \State $\vz'\sim p_0$; $a\sim\operatorname{Unif}(0,1)$
        \State \textbf{if }$a<\min\{1,\exp((R(\vz')-R(\vz))/\lambda)\}$\textbf{ then }$\vz_k\gets\vz'$
    \EndFor
    \State $\rho_k\gets\|\vz_k\|_2$; $\vu_k\gets\vz_k/\rho_k$
    \For{$n=1,\ldots,N$} \Comment{parallel reward queries}
        \State $\veps_k^{(n)}\sim\mathcal{N}(0,\mathbf{I})$
        \State $R_k^{(n)}\gets r\!\left(G_\theta(\sqrt{1-\sigma_k^2}\vz_k+\sigma_k\veps_k^{(n)})\right)$
    \EndFor
    \State $w_k^{(n)}\gets\operatorname{softmax}(R_k^{(n)}/\lambda)$
    \State $\widehat{\vg}_k\gets\dfrac{\lambda\sqrt{1-\sigma_k^2}}{\sigma_k}\sum_{n=1}^{N}w_k^{(n)}\veps_k^{(n)}$ \Comment{soft value gradient}
    \State $g_k^{\mathrm{rad}}\gets\vu_k^\top\widehat{\vg}_k$
    \State $\vg_k^{\mathrm{tan}}\gets\widehat{\vg}_k-g_k^{\mathrm{rad}}\vu_k$
    \State $\bm{\xi}_k\sim\mathcal{N}(0,\mathbf{I})$; $\bm{\xi}_k^{\mathrm{tan}}\gets\bm{\xi}_k-(\vu_k^\top\bm{\xi}_k)\vu_k$
    \State $\vu_{k+1}\gets\operatorname{Exp}_{\vu_k}(\eta_{A}\dfrac{\rho_k}{\lambda}\vg_k^{\mathrm{tan}}+\sqrt{2\eta_{A}}\bm{\xi}_k^{\mathrm{tan}})$ \Comment{spherical update}
    \State $\rho_{k+1}\gets\rho_k+\eta_{R}\left[\dfrac{d-1}{\rho_k}-\rho_k+\dfrac{1}{\lambda}g_k^{\mathrm{rad}}\right]+\sqrt{2\eta_{R}}\xi_k\textbf{ where }\xi_k\sim\mathcal{N}(0,1)$ \Comment{radial update}
    \State $\vz_{k+1}\gets\rho_{k+1}\vu_{k+1}$
\EndFor
\State \Return $G_\theta(\vz_K)$
\end{algorithmic}
\end{algorithm}

\subsection{Implementation of Zeroth-Order Baselines}

\textbf{ZO Langevin.} We substitute the zeroth-order gradient estimator $\widehat{\vg}(\vz_k)$ in Eq.~\eqref{eq:local_zo_estimator} into Eq.~\eqref{eq:grad_langevin} and apply an Euler-Maruyama step of size $\eta$:
\begin{equation}
    \vz_{k+1}=\vz_k+\eta\left(-\vz_k+\frac{1}{\lambda}\widehat{\vg}(\vz_k)\right)+\sqrt{2\eta}\,\bm{\xi}_k,\qquad \bm{\xi}_k\sim\gN(0,\rmI).
\end{equation}

\textbf{ZO Exponential Map.} Let $\rho_k=\|\vz_k\|_2$ and $\vu_k=\vz_k/\rho_k$. We project the same estimator onto the tangent space and update the direction through the spherical exponential map with step size $\eta$:
\begin{equation}
    \vg_k^{\mathrm{tan}}=(\rmI-\vu_k\vu_k^\top)\widehat{\vg}(\vz_k),\qquad
    \vu_{k+1}=\operatorname{Exp}_{\vu_k}\!\left(\frac{\eta\rho_k}{\lambda}\vg_k^{\mathrm{tan}}\right),\qquad
    \vz_{k+1}=\rho_k\vu_{k+1}.
\end{equation}

\textbf{Natural Evolution Strategies.}
NES~\citep{wierstra2014natural} maximizes the expected reward $J(\boldsymbol{\theta})=\mathbb{E}_{\vz\sim\pi_{\boldsymbol{\theta}}}[R(\vz)]$ by updating the parameters of a search distribution along its natural gradient:
\begin{equation}
    \boldsymbol{\theta}_{k+1}
    =\boldsymbol{\theta}_k+\eta\mathbf{F}(\boldsymbol{\theta}_k)^{-1}\widehat{\nabla}_{\boldsymbol{\theta}}J,
    \qquad
    \widehat{\nabla}_{\boldsymbol{\theta}}J
    =\frac{1}{N}\sum_{n=1}^{N}(R(\vz^{(n)})-\bar R)\nabla_{\boldsymbol{\theta}}\log\pi_{\boldsymbol{\theta}}(\vz^{(n)}),
\end{equation}
where $\mathbf{F}$ is the Fisher information matrix. Among the possible search distributions, we use an isotropic Gaussian $\pi_{\vz_k}=\gN(\vz_k,\sigma^2\rmI)$ with fixed $\sigma=\texttt{grad\_sigma}$ and optimize only its mean. For $\vz_k^{(n)}=\vz_k+\sigma\veps^{(n)}$, the resulting update is
\begin{equation}
    \vz_{k+1}
    =\vz_k+\eta\sigma^2\widehat{\vg}_k,
    \qquad
    \widehat{\vg}_k
    =\frac{1}{N\sigma}\sum_{n=1}^{N}\left(R(\vz_k+\sigma\veps^{(n)})-\bar R\right)\veps^{(n)},
    \qquad \veps^{(n)}\sim\gN(0,\rmI).
\end{equation}

\subsection{Text-to-Image Generation}
\label{subsec:image_setting}

\textbf{Models.} 
We utilized SDXL-turbo~\citep{sauer2024adversarial} and DMD2~\citep{yin2024improved} without Classifier-Free Guidance scale for our text-to-image generative models. All experiments for zeroth-order methods are conducted using NVIDIA RTX 4090 GPUs with VRAM 24GB. For first-order methods which require backpropagation through reward model and generator, NVIDIA A100 with VRAM 40GB was used.

\textbf{Reward function and text prompts.} ImageReward~\citep{xu2024imagereward} and PickScore~\citep{kirstain2023pickapic} return the human preference of a given image related to the text prompt. GenEval~\citep{ghosh2023geneval} evaluates the alignment between a given image and text prompt by aggregating detection result for text attributes.
For counting reward, we employ an external model T2ICount~\citep{qian2025t2icount} that estimates the count of given query objects in a given image. Here, to implement a challenging black-box reward scenario, we rounded the continuous T2ICount output to a nearest whole number before used by any initial noise optimization methods. 
OCR reward is computed using PaddleOCR~\citep{cui2025paddleocr} according to a normalized Levenshtein distance between the recognized and target text.

For counting reward task, we imported a quantity-aware generation prompt set from~\citet{yoon2026psisampler} which contains text prompts of various objects attached with a certain query number.
For other tasks, we used HPDv2~\citep{wu2023human} prompts for ImageReward, Pick-a-Pic~\citep{kirstain2023pickapic} prompts for PickScore and OCR, GenEval test prompts for GenEval, where we randomly selected 200 prompts.

\textbf{Hyperparameters.}
For the results in Table~\ref{tab:main_results}, we use $N=16$ particles per update for the particle-based zeroth-order optimization methods. The number of iteration was set to $K=25$ for ImageReward and Counting, and $K=50$ for PickScore, GenEval, and OCR. These settings correspond to $N\times K=400$ and $800$ number of reward function evaluations, respectively. 
First-order methods require one function evaluation and its backpropagation per iteration, so we allocated $N\times K$ iterations for each task. Note that ZeNOVA utilized $N_{\text{MH}}$ number of reward evaluations for MH jumping, so only $N-N_{\text{MH}}$ number of particles were used for value gradient estimation unlike other zeroth-order methods. We used $N_{\text{MH}}=4$ for all tasks.

We set the guidance scale to $1/\lambda=100.0$ for the Counting reward and $1/\lambda=1000.0$ for the rest of the reward function, unless specified. Step size for ZeNOVA was set to $\eta_A=0.1$ and $\eta_R=0.03$. For horizon scheduling, we fix $\sigma_1=0.9$ and $\sigma_K=0.1$, then use geometric interpolation to get the intermediate horizons so that the consecutive $\sigma$ have the same ratio.

\subsection{Text-to-Video Generation}
\label{subsec:video_setting}

\textbf{Model.}
We use rCM~\citep{rcm}, a few-step distilled text-to-video flow-matching model as our reference video generator. All experiments are conducted on a single NVIDIA RTX PRO 6000 Blackwell.

\textbf{Reward function and text prompts.} 
Following~\citet{xue2025dancegrpo}, we use VideoAlign~\citep{videoalign} to assess generated videos along three dimensions: Visual Quality (VQ), Motion Quality (MQ), and Text Alignment (TA). We randomly select 30 text prompts from VidProM~\citep{wang2024vidprom}. We conduct four separate optimization settings, using VQ, MQ, TA, or their sum (\emph{i.e.}, All) as the reward objective.

\textbf{Hyperparameters.}
We use $N=32$ particles per update and $K=15$ optimization iterations in all four reward settings, corresponding to $N\times K=480$ particle reward evaluations. The other hyperparameters are set as the same as in the text-to-image generation.



\section{Additional Results}
\label{sec:additional_exp}

Figures~\ref{fig:additional_imagereward}-\ref{fig:additional_video} show qualitative examples with various black-box rewards from different text-to-image generative models(SDXL-Turbo, DMD2) and the text-to-video generative model(rCM) with ZeNOVA.

\begin{figure}[t]
     \centering
     \includegraphics[width=0.95\linewidth]{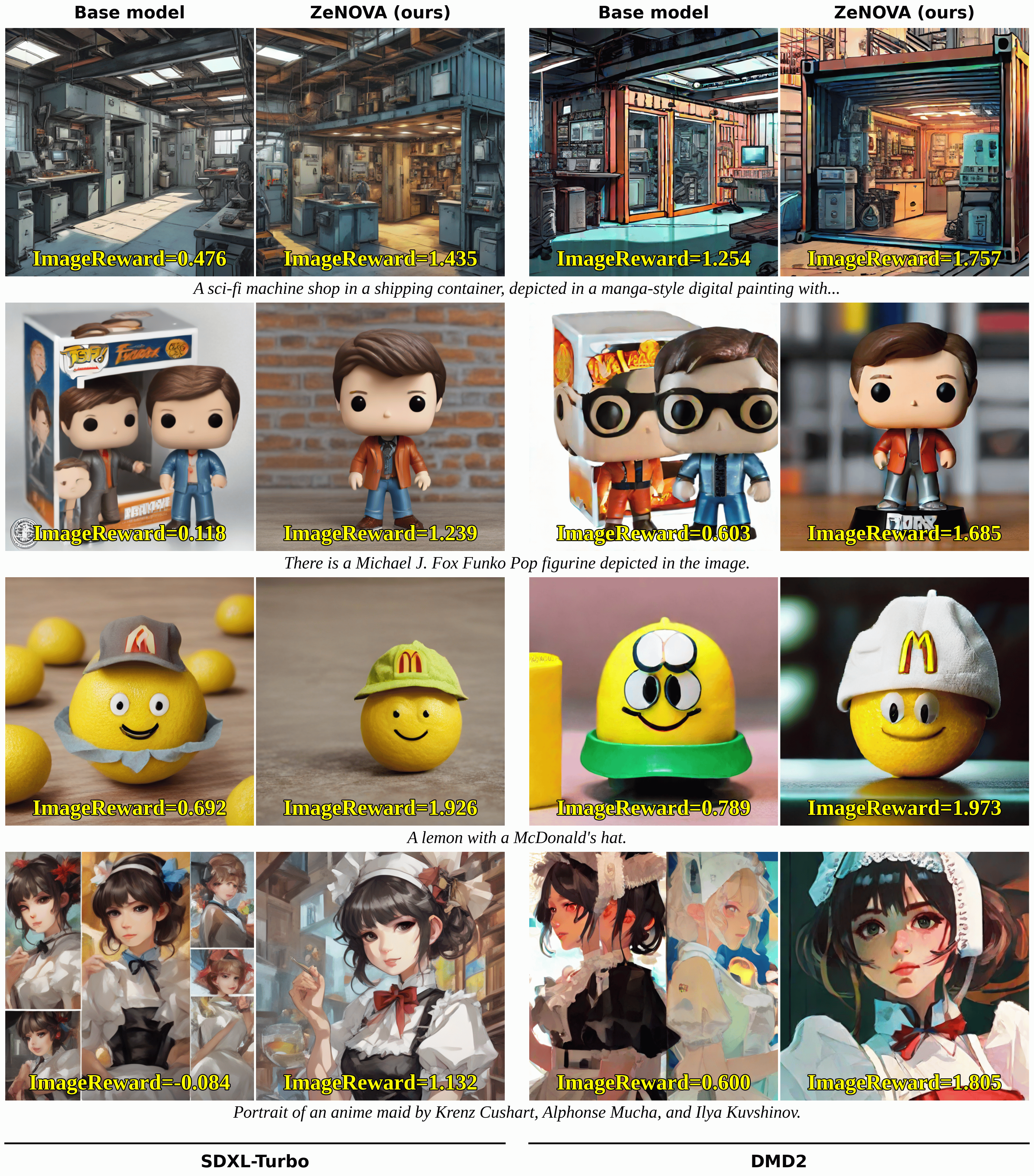}
     \caption{Images generated by SDXL-Turbo and DMD2 with ZeNOVA, using ImageReward as the reward signal.}
     \label{fig:additional_imagereward}
\end{figure}

\begin{figure}[t]
     \centering
     \includegraphics[width=0.95\linewidth]{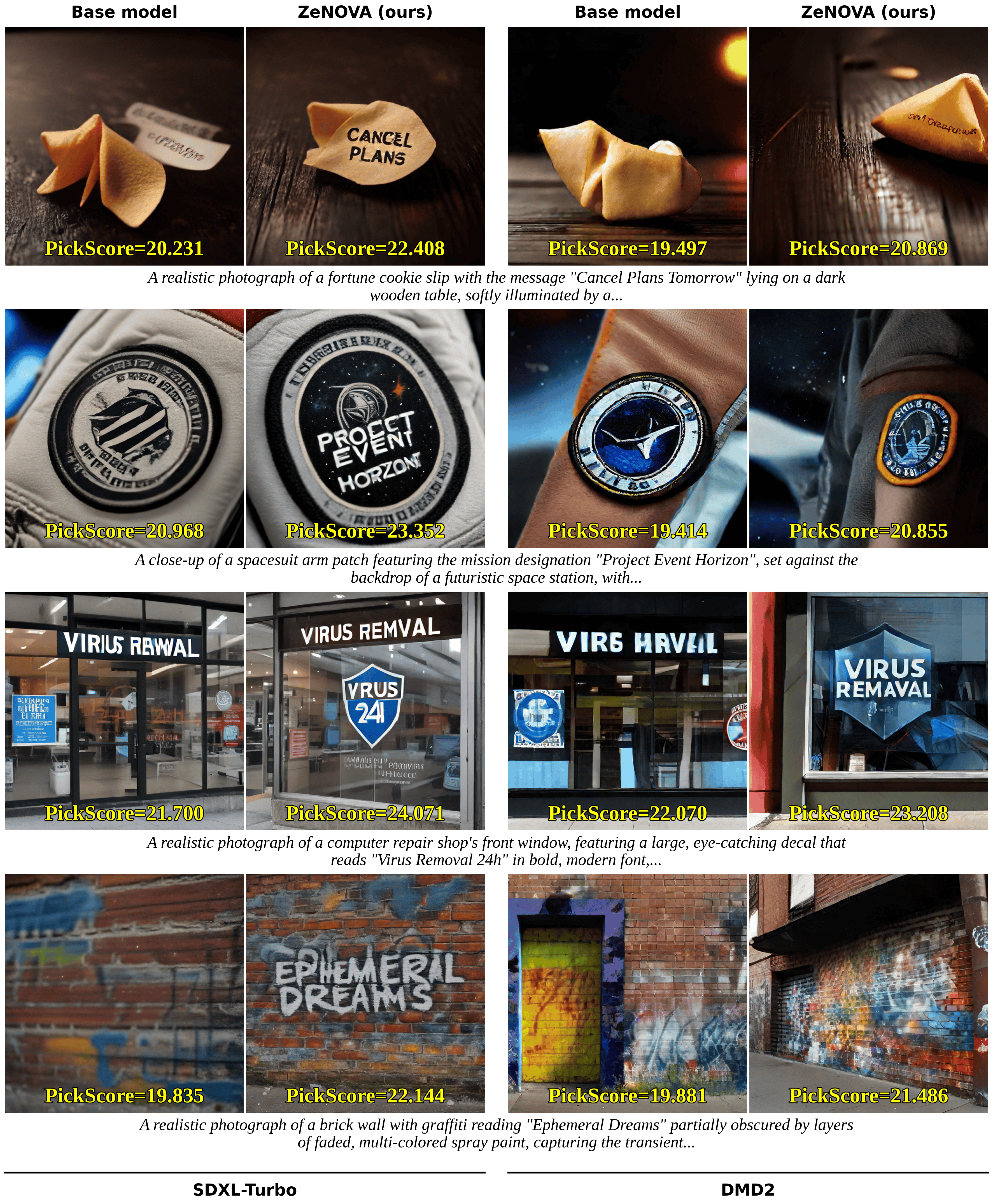}
     \caption{Images generated by SDXL-Turbo and DMD2 with ZeNOVA, using PickScore as the reward signal.}
     \label{fig:additional_pickscore}
\end{figure}

\begin{figure}[t]
     \centering
     \includegraphics[width=0.95\linewidth]{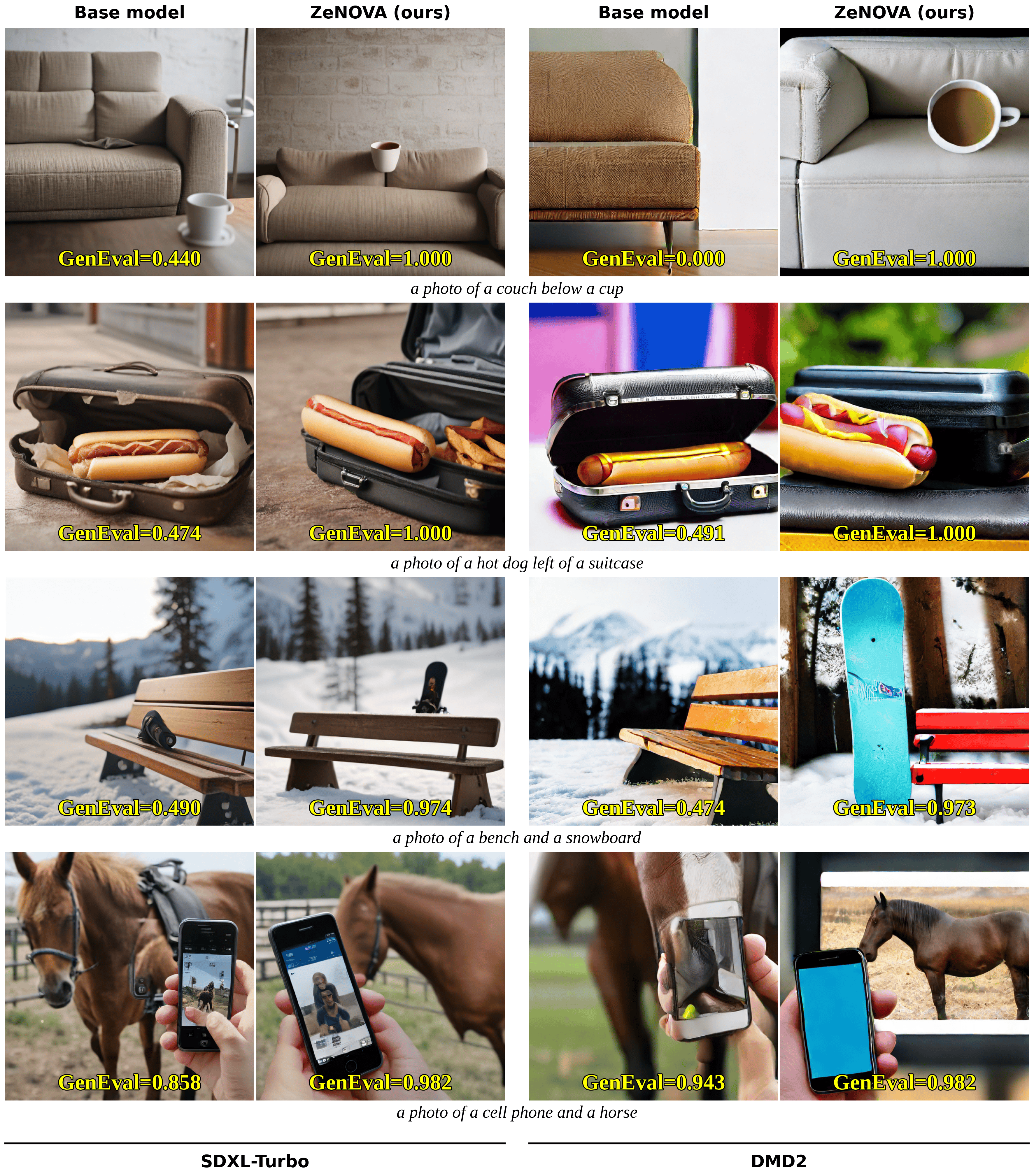}
     \caption{Images generated by SDXL-Turbo and DMD2 with ZeNOVA, using GenEval as the reward signal.}
     \label{fig:additional_geneval}
\end{figure}

\begin{figure}[t]
     \centering
     \includegraphics[width=0.95\linewidth]{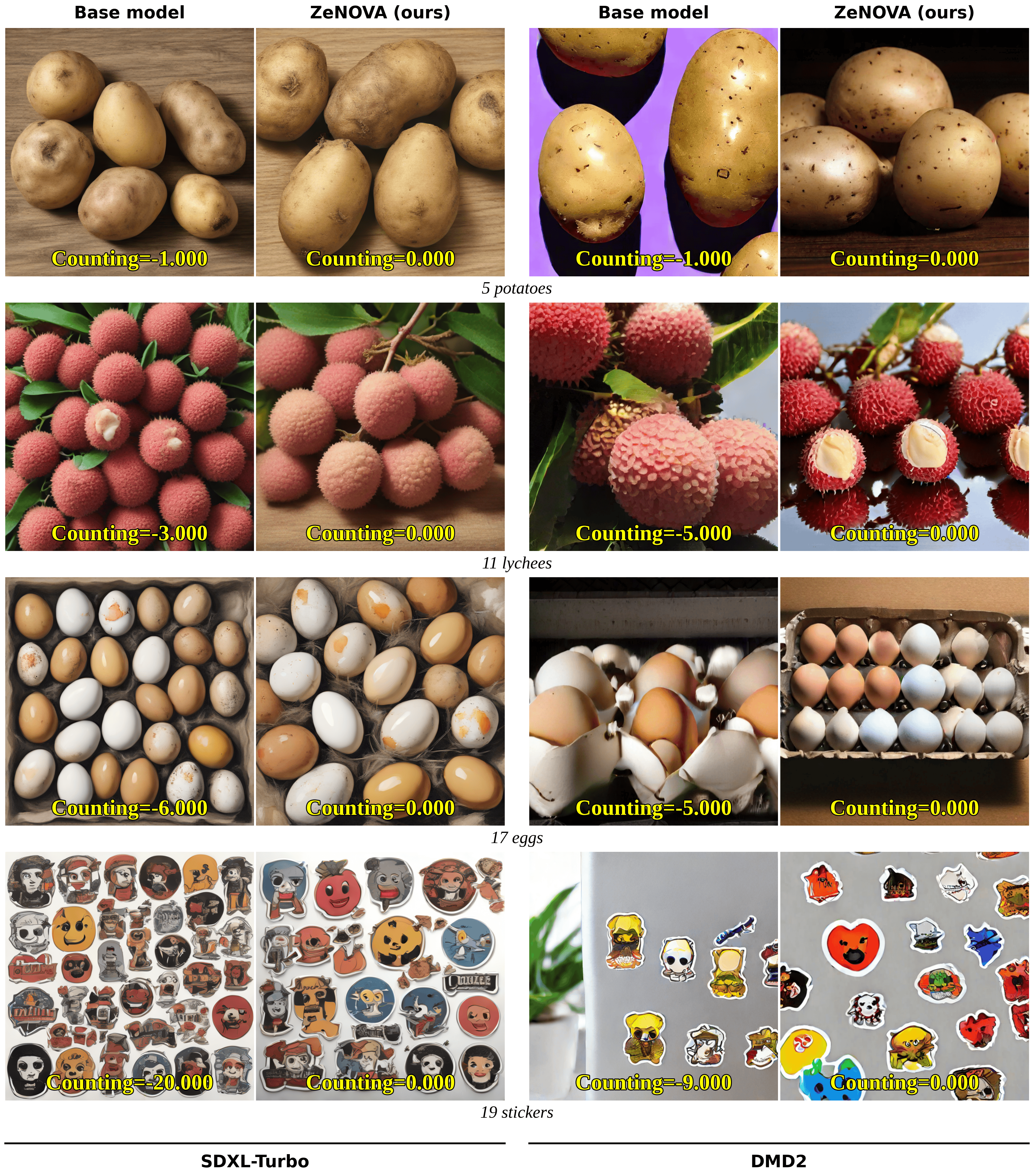}
     \caption{Images generated by SDXL-Turbo and DMD2 with ZeNOVA, using the counting reward as the reward signal.}
     \label{fig:additional_counting}
\end{figure}

\begin{figure}[t]
     \centering
     \includegraphics[width=0.95\linewidth]{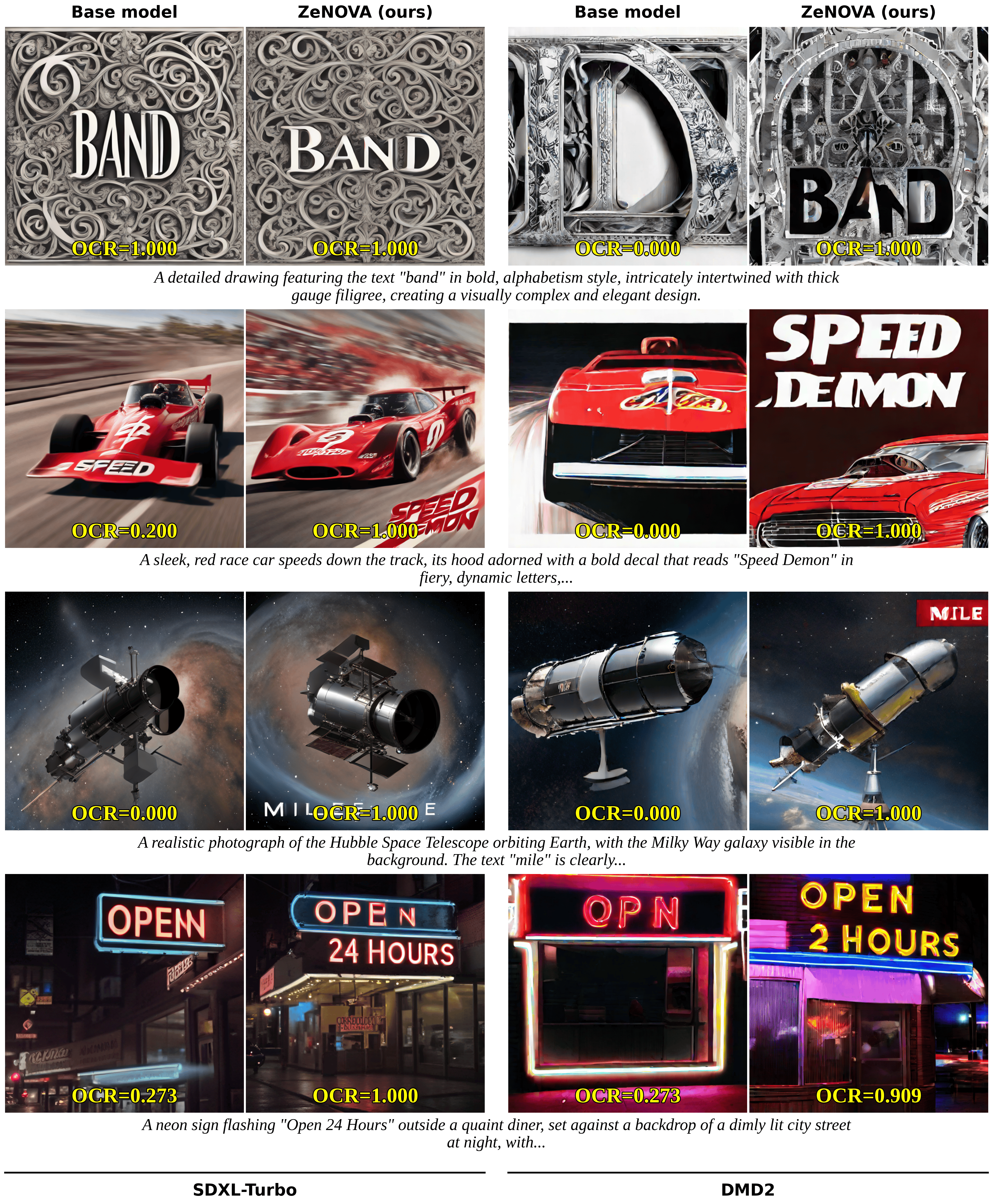}
     \caption{Images generated by SDXL-Turbo and DMD2 with ZeNOVA, using the OCR reward as the reward signal.}
     \label{fig:additional_ocr}
\end{figure}

\begin{figure}[t]
     \centering
     \includegraphics[width=\linewidth]{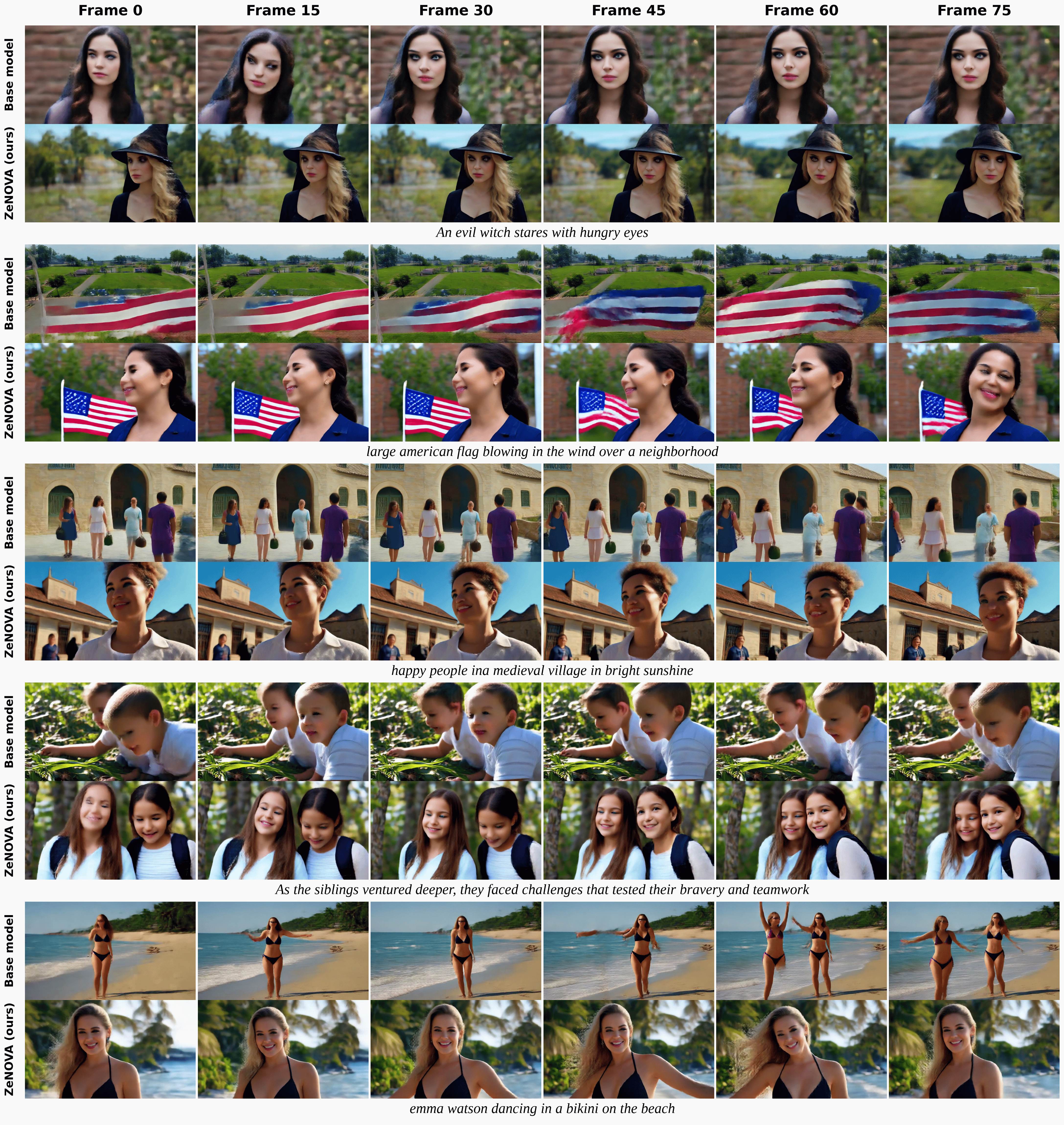}
     \caption{Vidoes generated by rCM with ZeNOVA, using a sum of three VideoAlign rewards as the reward signal.}
     \label{fig:additional_video}
\end{figure}

\end{document}